\documentclass[11pt]{article}
\usepackage[margin=1in]{geometry}

\usepackage[square,numbers,sort&compress]{natbib}
\usepackage[colorlinks, linkcolor=blue, filecolor = blue, citecolor = blue, urlcolor = blue]{hyperref}

\usepackage[T1]{fontenc}
\usepackage{lmodern}
\usepackage[utf8]{inputenc} 
\usepackage{url}            
\usepackage{booktabs}       
\usepackage{nicefrac}       
\usepackage{microtype}      
\usepackage{xcolor}         

\usepackage{amsmath, amsthm, amssymb, amsfonts}
\usepackage{bm}
\usepackage{mathtools}
\usepackage{tabularx}
\usepackage{thm-restate}
\usepackage{subfig}
\usepackage{color}
\usepackage{enumitem}
\usepackage{lipsum}

\usepackage[capitalize]{cleveref}

\usepackage[extdef=true]{delimset}
\usepackage{algorithm}
\usepackage{algorithmic}

\usepackage{mathtools}
\usepackage{thmtools}

\usepackage{microtype}
\usepackage{amssymb}
\usepackage{pifont}
\usepackage{nicefrac}
\usepackage{cancel}
\usepackage{array}
\usepackage{graphicx}
\usepackage{dsfont}
\usepackage{booktabs} 
\usepackage{amssymb}
\usepackage{color}
\usepackage{amsmath}
\usepackage{amsthm}
\usepackage{thmtools}
\usepackage{thm-restate}
\usepackage{algorithm}
\usepackage{algorithmic}
\usepackage{enumitem}
\usepackage{setspace}
\usepackage{thm-restate}
\usepackage{nameref}

\Crefname{equation}{Eq.}{Eqs.}
\Crefname{figure}{Fig.}{Figs.}
\Crefname{tabular}{Tab.}{Tabs.}
\Crefname{section}{Sec.}{Secs.}
\Crefname{appendix}{App.}{Apps.}
\Crefname{lemma}{Lem.}{Lems.}
\Crefname{theorem}{Thm.}{Thms.}
\Crefname{remark}{Rem.}{Rems.}
\Crefname{algorithm}{Alg.}{Algs.}
\Crefname{definition}{Def.}{Defs.}

\newtheorem{theorem}{Theorem}
\newtheorem*{theorem*}{Theorem}
\newtheorem*{lemma*}{Lemma}
\newtheorem{lemma}[theorem]{Lemma}
\newtheorem{rem}{Remark}

\newcommand{\R}{{\mathbb R}}

\renewcommand{\P}{{\mathbb P}}
\newcommand{\bP}{{\mathbf P}}
\newcommand{\bhP}{\mathbf {\hat P}}
\newcommand{\W}{{\mathbf W}}
\newcommand{\E}{{\mathbb E}}

\newcommand{\cE}{{\mathcal E}}

\newcommand{\cN}{{\mathcal N}}
\newcommand{\bN}{{\mathbf N}}

\newcommand{\cD}{{\mathcal D}}

\newcommand{\cR}{{\mathcal R}}

\newcommand{\tK}{{\tilde {K}}}

\newcommand{\hp}{{\hat p}}

\renewcommand{\r}{{\mathbf r}}

\newcommand{\sm}{\setminus}

\newcommand{\htheta}{{\hat \theta}}
\newcommand{\utheta}{\theta^{\ucb}}

\renewcommand{\epsilon}{\varepsilon}
\renewcommand{\hat}{\widehat}

\def \algupdt{{\it Rank-Break}}
\def \ucb{\texttt{ucb}}

\newcommand{\btheta}{\boldsymbol \theta}

\newcommand{\bSigma}{\boldsymbol \Sigma}

\newcommand{\bsigma}{\boldsymbol \sigma}

\DeclareMathOperator*{\argmax}{argmax}
\newcommand{\indic}{\mathds{1}}
\def\mystrut(#1,#2){\vrule height #1pt depth #2pt width 0pt} 

\def \prob{\text{Active Optimal Assortment}}

\def \pl{\text{PL}}
\def \aa{\text{AOA}}
\def \wf{\red{WI}}
\def \ff{\texttt{FR}}

\def \rmd{\relax} 

\def \aalgwf{\text{\aa-RB$_{\pl}$-Adaptive}} 
\def \algwf{\text{\aa-RB$_{\pl}$}} 

\def \regt{\textit{Reg}^{\texttt{top}}}
\def \regu{\textit{Reg}^{\texttt{wtd}}}
\def \topm{\text {Top-$m$}}
\def \utilm{\text {Wtd-Top-$m$}}
\def \papertitle{Stop Relying on No-Choice and Don't Repeat the Moves: Optimal, Efficient and Practical Algorithms for Assortment Optimization}

\newcommand{\red}[1]{\textcolor{red}{#1}}

\makeatletter
\renewcommand{\paragraph}{%
  \@startsection{paragraph}{4}%
  {\z@}{0.1ex \@plus .5ex \@minus .1ex}{-1em}%
  {\normalfont\normalsize\bfseries}%
}
\makeatother
\allowdisplaybreaks

\allowdisplaybreaks

\title{\papertitle}

\author{}
\author{
Aadirupa Saha%
\thanks{Apple. This work started when the author was at TTI, Chicago; {\tt aadirupa@ttic.edu}.}
\and 
Pierre Gaillard%
\thanks{Univ. Grenoble Alpes, Inria, CNRS, Grenoble INP, LJK, 38000 Grenoble, France.; {\tt pierre.gaillard@inria.fr}.} 
}
\date{}

\begin{document}

\maketitle

\begin{abstract}
We address the problem of active online assortment optimization problem with preference feedback, which is a framework for modeling user choices and subsetwise utility maximization. The framework is useful in various real-world applications including ad placement, online retail, recommender systems, fine-tuning language models, amongst many.
The problem, although has been studied in the past, lacks an intuitive and practical solution approach with simultaneously efficient algorithm and optimal regret guarantee. E.g., popularly used assortment selection algorithms often require the presence of a `strong reference' which is always included in the choice sets, further they are also designed to offer the same assortments repeatedly until the reference item gets selected---all such requirements are quite unrealistic for practical applications.
In this paper, we designed efficient algorithms for the problem of regret minimization in assortment selection with \emph{Plackett Luce} (PL) based user choices. We designed a novel concentration guarantee for estimating the score parameters of the PL model using `\emph{Pairwise Rank-Breaking}', which builds the foundation of our proposed algorithms. Moreover, our methods are practical, provably optimal, and devoid of the aforementioned limitations of the existing methods. 
%
Empirical evaluations corroborate our findings and outperform the existing baselines. 
\end{abstract}

\vspace{-15pt}
\section{Introduction}
\vspace{-5pt}
\label{sec:intro}

Studies have shown that it is often easier, faster and less expensive to collect feedback on a relative scale rather than asking ratings on an absolute scale. E.g., to understand the liking for a given pair of items, say (A,B), it is easier for the users to answer preference-based queries like: ``Do you prefer Item A over B?", rather than their absolute counterparts: ``How much do you score items A and B in a scale of  [0-10]?".
%
Due to the widespread applicability and ease of data collection with relative feedback, learning from preferences has gained much popularity in the machine-learning community, especially the active learning literature which has applications in Medical surveys, AI tutoring systems, Multi-player sports/games, or any real-world systems that have ways to collect feedback in terms of preferences. 

The problem is famously studied as the \emph{Dueling-Bandit} (DB) problem in the active learning community \cite{Yue+12,ailon2014reducing,Zoghi+14RUCB,Zoghi+14RCS,Zoghi+15}, which is an online learning framework for identifying a set of `good' items from a fixed decision-space (set of items) by querying preference feedback of actively chosen item-pairs. 

Consequently, the generalization of Dueling-Bandits, with {\em subset-wise} preferences has also been developed into an active field of research. For instance, in applications like web search, online shopping, recommender systems typically involve users expressing preferences by clicking on/choosing one result (or a handful of results) from a subset of offered items and often the objective of the system is to identify the `most-profitable' subset to offer to their users. The problem, popularly termed as `Assortment Optimization' is studied in many interdisciplinary literature, e.g. Online learning and bandits \cite{Busa21survey}, Operations research \cite{assort-consumer,assort-mnl}, Game theory \cite{chatterji21}, RLHF \cite{chatgpt,instructgpt}, to name a few. 


\textbf{Problem (Informal): \prob \,(\aa)}
Active Assortment Optimization (a.k.a. Utility Maximization with Subset Choices) \cite{assort-discrete,assort-mnl,assort-markov,assort-mallows} is an active learning framework for finding the `optimal' profit-maximizing subset. Formally, assume we have a decision set of $[K]:= \{1,2,\ldots K\}$ of $K$ items, with each item being associated with the score (or utility) parameters $\btheta:= (\theta_1,\theta_2,\ldots,\theta_K > 0 )$ (without loss of generality assume $\theta_1 \geq \theta_2 \geq \ldots \geq \theta_K$). 
At each round $t = 1,2,\ldots$, the learner or the algorithm gets to query an assortment (typically subsets containing up to $m$-items) $S_t \subseteq [K]$, upon which it gets to see some (noisy) relative preferences across the items in $S_t$, typically generated according to some underlying choice models based on the utilities $\btheta$. Further, to allow the occurrence of the event where no items are selected, we also model a No-Choice (NC) item, indexed by item-$0$, with its corresponding PL parameter $\theta_0 \in \R_+$.

\textbf{(Objective 1.) \topm-\aa:} One simplest objective could be to just identify the top-$m$ item-set: $\{\theta_1, \ldots, \theta_m\}$, for some $m \in [1,K]$. 

\textbf{(Objective 2.) \utilm-\aa:}
A more general objective could also consider a weight (or price) $r_i \in \R_+$ associated with the item $i \in [K]$, and the goal could be to identify the assortment (subset) with maximum weighted utility \footnote{This is equivalent to finding the set with maximum expected revenue when $r_i$s represents the price of item $i$ \cite{assort-mnl}}. Precise technical details are provided in \cref{sec:prob}.

\textbf{Related Works and Limitations: } As stated above, the problem of \aa\, is fundamental in many practical scenarios, and thus widely studied in multiple research areas, including Online ML/learning theory and operations research.   

\textbullet\, In the Online ML literature, the problem is well-studied as \emph{Multi-Dueling Bandits} \cite{Sui+17, Brost+16}, or Battling Bandits \cite{SG19,SGrank18,bengs2021preference}, which is an extension of the famous \emph{Dueling Bandit} problem \cite{Zoghi+14RCS,Zoghi+14RUCB}. 
The main limitation of this line of work is the lack of practical objectives, which either aim to identify the `best-item' $1 (= \arg\max_{i \in [K]}\theta_i$) within a PAC (probably approximately correct) framework \cite{SGwin18,ChenSoda+17,ChenSoda+18,Ren+18} or quantifying regret against the best items \cite{SG19,bengs2022stochastic}. Note the latter actually leads to the optimal subset choice of repeatedly selecting the optimal item, $\arg\max_{i}\theta_i$, $m$ times, i.e. $(1,1,\ldots1)$, which is highly unrealistic from the viewpoint of real-world system design. Selecting an assortment of distinct top-$m$ items (\topm-\aa) or maximum expected utility (\utilm-\aa) makes more sense. 

\textbullet\, On the other hand, a similar line of the problem has been studied in operations research and dynamic assortment selection literature, where the goal is to offer a subset of items to the customers in order to maximize expected revenue. The problem has been studied under different user choice models, e.g. PL or Multinomial-Logit models \citep{assort-mnl}, Mallows and mixture of Mallows \citep{assort-mallows}, Markov chain-based choice models \citep{assort-markov}, single transition model \citep{assort-stm} etc. While these works indeed consider a more practical objective of finding the best assortment (subset) with the highest expected utility for a regret minimization objective, (1) a major drawback in their approach lies in the algorithm design which \emph{requires to keep on querying the same set multiple times}, e.g. \cite{assort-mnl,ou2018multinomial,assort-nested}. Such design techniques could be impractical to be deployed in real systems where users could easily get annoyed if the same items are shown again and again. (2) The second major drawback of this line of work lies in the \emph{structural assumption of their underlying choice models which requires the existence of a reference/default item, that needs to be every part assortment} $S_t$. 
This leads to assuming a No-Choice item, typically denoted as item-$0$, which is a default choice of any assortment $S_t$.
Further a stronger and more unrealistic assumption lies in the fact that they require to assume that the above pivot is stronger than the rest of the $K$ items, i.e. $\theta_0 \geq \max_{i \in [K]} \theta_i$, i.e. the No-Choice (NC) action is the most likely outcome of any assortment $S_t$! This is rather unintuitive and often does not hold good in practice. The requirement of this assumption in the literature \cite{assort-mnl} primarily stems from a theoretical need: Without this, their proposed algorithms fail to maintain effective estimation and concentration bounds of the PL parameters $\btheta$. Of course, this is not a justifiable reason.

Considering the limitations of the above line of works in the existing literature \aa, we set to answer two questions: 
\begin{center}
	\emph{(1) Can we consider a general \aa\ model where the default item, like the NC item defined above, is not necessarily the strongest one, i.e. $\theta_0 \geq \max_{i \in [K]} \theta_i$?}
 
        \emph{(2) Can we still design a practical and regret optimal algorithm for this \aa\ framework, which need not have to play repetitive actions and yet converge to the optimal assortment efficiently?}
 \\
\end{center}

We answer the questions in the affirmative and present best of all scenarios: \emph{We design practical algorithms on practical \aa\, framework with practical objectives}--Unlike the existing approaches of the assortment optimization literature \cite{assort-mnl,assort-nested}, we do not have to keep playing the same assortment multiple times, neither require a strongest default item (like NC satisfying $\theta_0 \geq \max_{i \in [K]} \theta_i$). On the other hand, unlike the Online ML and Bandits literature, our objectives do not require us to converge to a multiset of replicated arms like $(1,1,\ldots1)$, but converge to the utility-maximizing set of distinct items. 

\textbf{Our Contributions:} 
\vspace{3pt}

\begin{enumerate}[ leftmargin=*, topsep=-\parskip]
\item \textbf{A General \aa\ Setup:} We work with a general problem of \aa\, for PL model, which requires no additional structural assumption of the $\btheta$ parameters such as $\theta_0 \geq \max_i \theta_i$, unlike the existing works. We designed algorithms for two separate objectives \topm\, and \utilm\, as discussed above.

\item \textbf{Practical, Efficient and Optimal Algorithm: } In \cref{sec:alg_wi}, we give a practical, efficient and provably optimal algorithm for MNL Assortment (up to logarithmic factors and the magnitude of $\theta_{\max}$ which is problem-dependent constant). The regret bound of our proposed algorithm \algwf\, (\cref{alg:wf}) yields $\tilde O(\sqrt{KT})$ regret for both \topm\, and \utilm\, objective. Our algorithms use a novel parameter estimation technique for discrete choice models based on the concept of \emph{Rank-Breaking} (RB) which is one of our key contributions towards designing the efficient and optimal algorithm. This enables our algorithm to not require the No-Choice item to be the strongest of the lot, unlike \cite{assort-mnl}.
\cref{lem:pl_simulator} details the key concept of our parameter estimation technique exploiting the concept of RB. Consequently, our resulting algorithm plays optimistically based on the UCB estimates of PL parameters and does not require repeating the same subset multiple times, justifying our title.  

\item \textbf{Further Improvement with Adaptive Pivots: } In \cref{sec:abi}, we further refined the regret analyses of \cref{alg:wf} emploring the novel idea of `adaptive pivots' and proposed \aalgwf. Performance-wise this removes the asymptotic dependence on $\theta_{\max} = \max_{i} \theta_i/\theta_0$ in the regret analysis. This enables the algorithm to work effectively in scenarios where the No-Choice item is less likely to be selected, i.e., $\theta_{\max} \gg 1$.  
This in fact, leads to a huge improvement in our algorithm performances, especially in the range of low $\theta_0$ where \aalgwf\, outperforms drastically over the existing baseline.


\item \textbf{Emperical Analysis. }  Finally, we corroborate our theoretical results with empirical evaluations  (Sec.\,\ref{sec:expts}), which certify the superior performance of our algorithm for the general \aa\ setups. 

\end{enumerate}

        \section{Problem Setup}
\label{sec:prob}


{\bf Notation.} We write the set $[n] = \{1,2,...,n\}$.
$\indic(\varphi)$ denotes an indicator variable that takes the value $1$ if the predicate $\varphi$ is true, and $0$ otherwise. 
$\P(A)$ is used to denote the probability of event $A$, in a probability space that is clear from the context.
The symbol $\lesssim$, employed in the proof sketches, represents a coarse inequality.

\subsection{Formulation}
\label{sec:setup}

We consider the sequential decision-making problem of \prob\ (\aa), with preference/choice feedback. %
Formally, the learner (algorithm) is given $[K]$, a finite set of $K$ items ($K>2$). At each decision round $t = 1, 2, \ldots$, the learner selects a subset $S_t \subseteq [K]$ of up to $m$ items, and receives some (stochastic) feedback about the item preferences of $S_t$, drawn according to some unknown underlying Plackett-Luce (PL) choice model with parameters $\btheta = (\theta_1,\theta_2,\ldots, \theta_K)$. An interested reader may check \cref{sec:RUM} for a detailed discussion on PL models, the precise formulation of the choice feedback is described in Section \ref{sec:feed_mod}. Given any assortment $S_t$ we also consider the possibility of `no-selection' of any items given an $S_t$. Following the literature of \cite{assort-mnl}, we model this mathematically as a No-Choice (NC) item, indexed by item-$0$, and its corresponding PL utility parameter $\theta_0$. 

\begin{rem}
It is important to note that, unlike the existing literature on assortment selection, we are not required to assume the NC-item to be the strongest i.e. $\theta_0 \not\geq \max_{i \in [K]}\theta_i$, or NC is not the most likely outcome assortment. Further, since the PL model is scale independent, without loss of generality, we always set $\theta_0 = 1$ and scale the rest of the PL parameters around that.
\end{rem}
We assume $\theta_1 \geq \theta_2 \geq \ldots \geq \theta_K$ without loss of generality. 
 

\subsection{Feedback model}
\label{sec:feed_mod}
The feedback model formulates the information received (from the `environment') once the learner plays a subset $S_t \subseteq [K]$ of at most $m$ items. Given $S_t$ we consider the algorithm receives a winner feedback (or index of an item) $i_t \in S_t \cup \{0\}$, drawn according to the underlying PL choice model 
%
as:
\begin{align}
\label{eq:prob_win}
\P(i_t = i|S_t) = \frac{{\theta_i}}{\theta_0 + \sum_{j \in S} \theta_j} ~~\forall i \in S_t.
\end{align} 

\subsection{Performance Objectives: } 
\label{sec:obj}

We consider the following two objectives: 

\textbf{Objective 1. \topm-\aa\, Regret:} One simplest objective could be to just identify the top-$m$ item-set: $\{\theta_1, \ldots, \theta_m\}$, for some $m \in [1,K]$. Formally the performance objective of the learner can be captured through the following regret minimization objective: 
\begin{align*}
    \regt_T:=  \sum_{t = 1}^T(\Theta_{S^*} - \Theta_{S_t})/{m},
\end{align*}
where $S^*:= \argmax \Theta_S$ is the (maximum) PL-Choice utility $\Theta_S := \sum_{i\in S} \theta_i$ of the set of top-$m$ items.

\textbf{Objective 2. \utilm-\aa\, Regret:}
In this case, we consider a more general setting, where each item-$i$ might is also associated with a weight (for example price) $r_i \in \R_+$, and the goal is to identify the set of size at most $m$ with maximum weighted utility. One could measure the regret of the learner as the utility difference between the optimal set and cumulative utility of the learner:
\begin{align*}
        \regu_T:= \sum_{t = 1}^T (\cR (S^*, \btheta) -\cR(S_t, \btheta)),
\end{align*}
where for any $S$, we define 
by 
\begin{equation}
    \cR(S,\btheta):= \sum_{i \in S}\frac{r_i\theta_i}{\theta_0 + \sum_{j \in S}\theta_j}
    \label{eq:utilobjective}
\end{equation} the weighted PL-choice utility of the subset $S$, and $S^*:= \argmax_{S \subseteq [K] \mid |S|\leq m}\cR(S, \btheta)$ denotes the optimal utility-maximzing subset. 


 
\section{\aa\ with PL: A Practical and Efficient Algorithm for $\regt$\, and ~$\regu$\, Objective}
\label{sec:alg_wi}

In this section, we propose our first algorithm for the two objectives: \topm \ and \utilm. The crux of our novelty lies in our PL parameter estimation technique which shows how one can maintain an estimate of pairwise scores of $p_{ij} = \frac{\theta_i}{\theta_i + \theta_j}$ for each pair of item $(i,j)$, $i,j \in [K]\cup\{0\}$ using the idea of \emph{Rank-Breaking} (RB). We discussed the details of the RB idea in \cref{sec:rb}. Combining the concept of RB along with a key lemma of \cite{SG19}, we estimate the PL parameters $\btheta$ and manage to maintain a tight concentration interval for each $\theta_i, ~i \in [K]$. Given these reasonable estimates of the PL parameters, we play the assortments $S_t$ optimistically at each round $t$. The key algorithmic ideas are detailed below. We analyze its regret guarantees for both \topm\ and \utilm \ objective, respectively in \cref{thm:topm_wf} and \cref{thm:utilm_wf}.

\subsection{Algorithm Design}
\label{subsec:alg_wi_detail}


\textbf{Maintaining Pairwise Estimates. }
At each time $t$, our algorithm maintains a pairwise preference matrix $\bhP_t \in [0,1]^{n \times n}$, whose $(i,j)$-th entry $\hp_{ij,t}$ records the empirical probability of $i$ having beaten $j$ in a pairwise duel, and a corresponding upper confidence bound $p^{\text{ucb}}_{ij,t}$ for each pair $(i,j)$, for each pair $(i,j) \in [\tK]\times[\tK]$, where $[\tK]:= [K]\cup\{0\}$.
We estimate $\hp_{ij,t}:= \frac{w_{ij,t}}{n_{ij,t}}$, where $w_{ij,t} = \sum_{s=1}^{t-1} \indic\{i_s = i,j \in S_s\}$ denotes the number of pairwise wins of item-$i$ over $j$ upon \emph{rank breaking} (RB), and $n_{ij,t} = w_{ij,t}+w_{ji,t}$ being the number of times $(i,j)$ has been compared upon RB. 
The algorithm further maintains UCB estimates, $p^{\text{ucb}}_{ij,t}$ of each $(i,j)$ pair, defined as 
    \begin{equation}
    p^{\text{ucb}}_{ij,t}:= \hp_{ij,t} + \sqrt{\frac{2 \hat p_{ij,t}(1-\hat p_{ij,t}) x}{n_{ij,t}}} + \frac{3x}{n_{ij,t}}.
    \label{eq:pucb_def}
    \end{equation}
Further by deriving intuition from \cref{lem:pl_simulator} we also establish a tight concentration of the true pairwise preferences $p_{ij} = \frac{\theta_i}{\theta_i + \theta_j}$ in terms of $p^{\text{ucb}}_{ij,t}$.
    
\textbf{Estimate upper-confidence-bounds $\theta_t^{\ucb}$ from Pairwise Estimates. } 
The above UCB estimates $p^{\text{ucb}}_{ij,t}$ are further used to design UCB estimates of the PL parameters $\theta_i$ as follows
    \[
      \theta_{i,t}^{\text{ucb}}  = \frac{p_{i0,t}^{\text{ucb}}}{(1 - p_{i0,t}^{\text{ucb}})_+}.
    \]

\textbf{Optimistic Assortment Selection:}    
The above estimates of $\theta_{i,t}^{\text{ucb}}$s are further used to select the set $S_t$, that maximizes the underlying objective. This optimization problem transforms into a static assortment optimization problem with upper confidence bounds $\theta_{i,t}^{\ucb}$ as the parameters, and efficient solution methods for this case are available (see e.g., \cite{avadhanula2016tightness,davis2013assortment, rusmevichientong2010dynamic}).

\begin{center}
\begin{algorithm}[h]
   \caption{\textbf{\aa\ for PL model with RB (\algwf)}}
   \label{alg:wf}
\begin{algorithmic}[1]
    \STATE {\bfseries input:} $x >0$
   \STATE {\bfseries init:} $\tK \leftarrow K+1$, $[\tK]= [K] \cup\{0\}$, $\W_1 \leftarrow [0]_{\tK \times \tK}$ 
   \FOR{$t = 1,2,3, \ldots, T$} 
	\STATE Set $\bN_t = \W_t + \W_t^\top$, and $\hat{\bP}_t = \frac{\W_t}{\bN_t}$. Denote $\bN_t = [n_{ij,t}]_{\tK \times \tK}$ and $\hat \bP_t = [\hp_{ij,t}]_{\tK \times \tK}$.
        \STATE Define for all $i$, $p^{\text{ucb}}_{ii,t} = \frac{1}{2}$ and for all $i,j \in [\tK], i\neq j$
        \[
        \textstyle \smash{p^{\text{ucb}}_{ij,t} = \hp_{ij,t} + \Big(\frac{2 { \hp_{ij,t}(1-\hp_{ij,t})} x}{n_{ij,t}}\Big)^{1/2} + \frac{3x}{n_{ij,t}}}\]
        \STATE $\utheta_{i,t}:= p^{\ucb}_{i0,t}/(1-p^{\ucb}_{i0,t})_+ $
        \STATE $S_{t} \leftarrow 
			\begin{cases}
				\text{Top-$m$ items from argsort}(\{\utheta_{1,t},\dots,\utheta_{K,t}\}),\\
    ~~~~~~~~~~~~~~~~~~~~~~~\text{ for } \topm\ \text{ objective }\\
				\argmax_{S \subseteq [K] \mid |S| \leq  m} \cR(S, \theta_t^{\ucb}),\\
    ~~~~~~~~~~~~~~~~~\text{ for } \utilm\ \text{ objective } 
	\end{cases}$
	\STATE Play $S_t$
        \STATE Receive the winner \ $i_t \in [\tK]$ (drawn as per \eqref{eq:prob_win}) 
   {\STATE \label{line:mm_rb} Update:  $\W_{t+1} = [w_{ij,t+1}]_{\tilde K \times \tilde K}$ s.t. $w_{i_tj,t+1} \leftarrow w_{i_tj,t} + 1 ~~ \forall j \in S_t\cup\{0\}$} 
   \ENDFOR
\end{algorithmic}
\end{algorithm}
\end{center}



\subsection{Analysis: Concentration Lemmas}
We start the analysis by providing three technical lemmas that provide confidence bounds for the $p_{ij}$ and $\theta_i$. The proofs are based on Bernstein concentration bounds and are deferred to the appendix.

\begin{lemma} 
\label{lem:pconc}
Let $(i,j) \in [K]\times[K]$. Let $T\geq 1$ and $\delta >0$. Then, with probability at least $1- 3 Te^{-x}$, 
\begin{equation}
    p_{ij} \leq p_{ij,t}^{\ucb} \leq  p_{ij} + 2\sqrt{\frac{2 p_{ij}(1-p_{ij}) x}{n_{ij,t}}} + \frac{11 x}{n_{ij,t}} \,, 
    \label{eq:ucb_p1}
\end{equation}
simultaneously for all $t \in [T]$.
\end{lemma}

\begin{lemma} 
\label{lem:tconc}
Let $T \geq 1$ and $x >0$. Then, with probability at least $1- 3KT e^{-x}$, then simultaneously for all $t \in [T]$ and $i \in [K]$: $\theta_i \leq \theta^{\text{ucb}}_{i,t}$ and one of the following two inequalities is satisfied
\[
    n_{i0,t} < 69 x (\theta_0 +\theta_i)
\]
or 
\[
\theta^{\text{ucb}}_{i,t} \leq \theta_i +  4(\theta_0+\theta_i)\sqrt{\frac{2 \theta_0 \theta_i x}{n_{i0,t}}} + \frac{22x(\theta_0+\theta_i)^2}{n_{i0,t}} \,.
\]
\end{lemma}

The above lemma depends on $n_{i0,t}$ the number of times items $i$ have been compared with item $0$ up to round $t$. The latter is controlled using the following lemma: 

\begin{lemma} \label{lem:nibound}
Let $T \geq 1$ and $x >0$. Then, with probability at least $1-KTe^{-x}$
\begin{equation}
    \label{eq:nibound}
    \tau_{i,t} < 2 x(\theta_0 + \Theta_{S^*})^2 \  \text{ or } \  n_{i0,t} \geq \frac{(\theta_0 + \theta_i) \tau_{i,t}}{2(\theta_0 + \Theta_{S^*})}\,,
\end{equation}
simultaneously for all $t \in [T]$ and $i \in [K]$.
\end{lemma}

\subsection{Analysis: \topm\ Objective:}
\label{sec:alg_wi_analys}

We are now ready to provide the regret upper bound for Algorithm~\ref{alg:wf} with \topm{} objective. 

\begin{restatable}[\algwf: Regret Analysis for \topm\, Objective]{theorem}{topmwf}
\label{thm:topm_wf}
Let $\theta_{\max} \geq 1$. Consider any instance of PL model on $K$ items with parameters $\theta \in [0,\theta_{\max}]^K$, $\theta_0 = 1$. The regret of \algwf{} (Alg.~\ref{alg:wf}) with parameter $x = 2\log T$ is bounded as 
\[
    \regt_T = O\big(\theta_{\max}^{3/2}\sqrt{KT \log T}\big) \quad \text{when } T \to \infty\,.
\]

\end{restatable}

Noting $\theta_{\max}$ is a problem-dependent constant, the above rate of $\tilde O(KT)$ is optimal (up to log-factors), as a lower bound can be straightforwardly derived from standard multi-armed bandits~\cite{Auer00,Auer+02}.
We only state here a sketch of the proof of \cref{thm:topm_wf}. The detailed proof is deferred to the appendix.

\begin{proof}[Proof Sketch of \cref{thm:topm_wf}]
Let us define for any $S \subseteq [K]$,
\[
\Theta_S = \sum_{i \in S}\theta_i, ~~\text{ and }~~ \Theta_S^{\ucb}:=  \sum_{i \in S}\theta_i^{\ucb}.
\]
Let $\cE$ be the high-probability event such that both Lemma~\ref{lem:tconc} and~\ref{lem:nibound} holds true. Then, $\P(\cE) \geq 1-4TKe^{-x}$. Let us first assume that $\cE$ holds true. Then, by Lemma~\ref{lem:tconc}, $\Theta_{S^*} \leq \Theta_{S^*}^{\ucb} \leq \Theta_{S_t}^{\ucb}$, which yields
\begin{align*}
\regt_T & = \frac{1}{m} \sum_{t=1}^T \Theta_{S^*} - \Theta_{S_t}  \\
    & \leq  \frac{1}{m} \sum_{t=1}^T  \Theta_{S_t}^{\ucb} - \Theta_{S_t}  \\
    & \leq O(\log T) + \frac{1}{m} \sum_{t=1}^T \sum_{i \in S_t} (\theta_{i,t}^{\ucb} - \theta_{i}) \indic\big\{ \tau_{i,t} \geq \tau_0 \big\}  
\end{align*}
where $\tau_0 = 138 x (m+1)^2 \theta_{\max}^2$ corresponds to an exploration phase needed for the confidence upper bounds of Lem~\ref{lem:tconc} and~\ref{lem:nibound} to be satisfied and $O(\log T)$ is the cost of that exploration. 

Then, noting that if $\cE$ holds true, we can show by Lemma~\ref{lem:nibound}, that 
\[
    \indic\{ \tau_{i,t} \geq \tau_0 \} \leq \indic\{ n_{i0,t} \geq 69 x (\theta_0 + \theta_i) \}.
\]
Therefore, we can apply Lemma~\ref{lem:tconc} that entails,
\begin{align*}
\frac{1}{m} \sum_{t=1}^T &  \sum_{i \in S_t} (\theta_{i,t}^{\ucb} - \theta_{i}) \indic\big\{ \tau_{i,t} \geq \bar n_{i0} \big\} \\
     & \lesssim    \frac{1}{m} \sum_{t=1}^T \sum_{i \in S_t}\Big( (\theta_0+\theta_i)\sqrt{\frac{ \theta_0 \theta_i x}{n_{i0,t}}} \indic\big\{  \tau_{i,t} \geq  \tau_0  \big\}  \\
    & \stackrel{\text{Lem.~\ref{lem:nibound}}}{\lesssim}   \frac{1}{m} \sum_{t=1}^T \sum_{i \in S_t}  \theta_{\max}^{3/2} \sqrt{\frac{m x}{\tau_{i,t}}}  \\
    & \lesssim  \frac{1}{m} \sum_{i=1}^K   \theta_{\max}^{3/2} \sqrt{m x \tau_{i,t}} \,,\\
    & \lesssim \theta_{\max}^{3/2} \sqrt{ x KT } \,.
\end{align*}
where we used $\sum_{i=1}^n 1/\sqrt{i} \leq 2 \sqrt{n}$ and $\sum_{i} \tau_{i,t} = m T$ together with Jensen's inequality in the last inequality. We thus have under the event $\cE$ that
\begin{align*}
\regt_T & \leq  O(\theta_{\max}^{3/2} \sqrt{ x KT })
\end{align*}
The proof is concluded by taking the expectation and by setting $x = 2 \log T$ to control $\P(\cE^c)$. 
\end{proof}

\subsection{Analysis: \utilm\ Objective}
\label{sec:alg_wiutil_analys}

We turn now to the analysis of the \utilm{} objective that aims to optimize $\cR(S, \theta)$ as defined in~\eqref{eq:utilobjective}.
We start by stating a lemma from \cite{assort-mnl} that shows that the expected utility $\cR(S^*, \theta)$ that corresponds to the optimal assortment $S^* = \argmax_{S \subset [K], |S| \leq m} \cR(S,\theta)$ is non-decreasing in the parameters $\theta$. 

\begin{lemma}[Lemma A.3 of \cite{assort-mnl}]
\label{lem:wtd_util}
If $\theta^{\ucb} \in \R^K$ is such that $\theta_i^{\ucb} \geq \theta_i$ for all $i \in [K]$, we have $\cR(S^*, \theta) \leq \cR(S^*, \theta^{\ucb})$.
\end{lemma}

\begin{restatable}[\algwf: Regret Analysis for \utilm\, Objective]{theorem}{utilmwf}
\label{thm:utilm_wf}
Let $\theta_{\max} \geq 1$. For any instance of the PL model on K items with parameter $\theta \in [0,\theta_{\max}]^K$ and weights $\r \in [0,1]^K$, the weighted regret of \algwf\ (Alg.~\ref{alg:wf}) with $x = 2 \log T$ is bounded as 
\[
    \regu_T = O(\sqrt{\theta_{\max} KT} \log T)
\]
as $T \to \infty$. 
\end{restatable}

The rate $\Omega(KT)$ is optimal as proved by the lower bound established by~\cite{chen2017note} for MNL bandit problems.

\begin{rem}
Our result recovers (up to a factor $\sqrt{\log T}$) the one of \cite{assort-mnl} when $\theta_{\max} = 1$. However, their algorithm relies on more sophisticated estimators that necessitate epochs repeating the same assortment until the No-Choice item is selected. Note for our problem setting, where it is possible to have $\theta_{\max} \gg \theta_0 = 1$, the length of these epochs could be of $O(K \theta_{\max})$, which could be potentially very large when $\theta_{\max} \gg 1$. This consequently reduces the number of effective epochs, leading to poor estimation of the PL parameters. We see this tradeoff in our experiments (\cref{sec:expts}) where the MNL-UCB algorithm of \cite{assort-mnl} yields linear $O(T)$ regret for such choice of the problem parameters. 
\end{rem}

We provide below the sketch of the proof of the theorem. The complete proof is postponed to the appendix.

\begin{proof}[Proof sketch of Thm.~\ref{thm:utilm_wf}]
Let $\cE$ be the high-probability event such that both Lemma~\ref{lem:tconc} and~\ref{lem:nibound} are satisfied.  Then, 
\begin{align}
    & \regu_T  
     = \sum_{t=1}^T \E\big[\cR(S^*,\theta) - \cR(S_t,\theta)\big] \nonumber \\
    & \lesssim \sum_{t=1}^T \E\big[(\cR(S^*,\theta) - \cR(S_t,\theta)) \indic\{\cE\}\big] + T \P(\cE^c)  \nonumber \\
    & \lesssim \sum_{t=1}^T \E\big[(\cR(S_t,\theta_t^{\ucb}) - \cR(S_t,\theta)) \indic\{\cE\}\big] + T \P(\cE^c) \label{eq:wtd1}
\end{align}
because $\cR(S_t,\theta_t^{\ucb}) \geq \cR(S^*, \theta_t^{\ucb}) \geq \cR(S^*,\theta)$ under the event $\cE$ by Lemma~\ref{lem:wtd_util}.
We now upper-bound the first term of the right-hand-side
\begin{align}
 \sum_{t = 1}^T&  \E\Big[  \Big(\big( \cR(S_t,\theta_t^{\ucb}) - \cR(S_t,\theta) \big)\Big) \indic\{\cE\} \Big] \nonumber \\
     & = \sum_{t=1}^T  \E\bigg[    \bigg( \sum_{i\in S_t} \frac{r_i \theta_{i,t}^{\ucb} }{\theta_0 + \Theta_{S_t,t}^\ucb} - \frac{r_i \theta_{i} }{\theta_0 + \Theta_{S_t}} \bigg)  \indic\{\cE\} \bigg] \nonumber \\
    & \leq  \sum_{t=1}^T  \E\bigg[    \bigg( \sum_{i\in S_t} \frac{r_i (\theta_{i,t}^{\ucb} - \theta_{i}) }{\theta_0+ \Theta_{S_t}} \bigg)  \indic\{\cE\} \nonumber
\end{align}
Because $\Theta_{S_t,t}^{\ucb} \geq \Theta_{S_t}$ under the event $\cE$ by Lemma~\ref{lem:tconc}. Then, using $r_i \leq 1$, we further upper-bound using an exploration parameter $\tau_0 = O(\log (T))$ so that the upper-confidence-bounds in Lemmas~\ref{lem:tconc} and~\ref{lem:nibound} are satisfied
\begin{align}
& \sum_{t = 1}^T  \E\Big[  \Big(\big( \cR(S_t,\theta_t^{\ucb}) - \cR(S_t,\theta) \big)\Big) \indic\{\cE\} \Big] \nonumber \\
    & \leq   \sum_{i=1}^K \E\Bigg[   \sum_{t=1}^T \bigg(   \frac{ | \theta_{i,t}^{\ucb}- \theta_{i}| }{\theta_0 + \Theta_{S_t}}  \bigg) \indic\{i \in S_t, \cE\}   \Bigg] \nonumber \\
    & \lesssim  O(\tau_0) \nonumber \\
    & \qquad + \sum_{i=1}^K \E\Bigg[   \sum_{t=1}^T   \frac{ | \theta_{i,t}^{\ucb}- \theta_{i}| }{\theta_0 + \Theta_{S_t}}  \indic\{i \in S_t, \tau_{i,t} \geq \tau_0 ,\cE\}   \Bigg] \nonumber \\
    & \lesssim O(\tau_0) +  \sum_{i=1}^K \sqrt{ \sum_{t=1}^T \E \Bigg[ \frac{ \theta_{i} \indic\{i \in S_t\}}{\theta_0 + \Theta_{S_t}} \Bigg] } \times   \nonumber \\
    &  \sqrt{\smash{\underbrace{\sum_{t=1}^T \E \Bigg[ \bigg( \frac{\theta_{i,t}^{\ucb}- \theta_{i} }{\theta_0 + \Theta_{S_t}}\bigg)^2\frac{\theta_0 + \Theta_{S_t}}{\theta_i} \indic\{i \in S_t, \tau_{i,t} \geq \tau_0, \cE\} \Bigg] }_{=: A_T(i)} } \mystrut(20,10) } \mystrut(20,24) \label{eq:longeq1}
\end{align}
where the last inequality is by Cauchy-Schwarz inequality. 
Now, the term $A_T(i)$ above may be upper-bounded using Lemmas~\ref{lem:tconc} and~\ref{lem:nibound},
\begin{align*}
  & A_T(i)  = \E \Bigg[  \frac{ (\theta_{i,t}^{\ucb}- \theta_{i})^2 }{ \theta_i(\theta_0 + \Theta_{S_t})  }  \indic\{i \in S_t, \tau_{i,t} \geq \tau_0, \cE\} \Bigg] \\ 
  & \lesssim  \sum_{t=1}^T \E \Bigg[  \frac{ (\theta_0 +\theta_{i})^2 x }{n_{i0,t} (\theta_0 + \Theta_{S_t})}  \indic\{i \in S_t \} \Bigg] \\
    & \lesssim  \theta_{\max} x   \sum_{t=1}^T \E \Bigg[ \frac{ (\theta_0 + \theta_i) \indic\{i \in S_t\}}{(\theta_0 + \Theta_{S_t}) n_{i0,t}}  \Bigg] \\
    & = \theta_{\max} x  \E \Bigg[ \sum_{t=1}^T \frac{  \indic\{i_t \in \{i,0\}, i\in S_t\}}{n_{i0,t}} \Bigg]\\
    & \lesssim    \theta_{\max} x   \log T
\end{align*}
where in the last inequality we used that $\sum_{n=1}^T n^{-1} \leq 1+ \log T$. 
Substituting into~\eqref{eq:longeq1}, Jensen's inequality entails,
\begin{multline}
\sum_{t = 1}^T   \E\Big[  \big( \cR(S_t,\theta_t^{\ucb}) - \cR(S_t,\theta) \big) \indic\{\cE\} \Big]  \lesssim O(\tau_0) + \E \Bigg[\sqrt{ \theta_{\max}x  \log T} \sum_{i=1}^K \sqrt{ \sum_{t=1}^T \frac{  \theta_{i} \indic\{i \in S_t\}}{\theta_0 + \Theta_{S_t}}  } \Bigg]\,.
\label{eq:wtd2}
\end{multline}
The proof is finally concluded by applying Cauchy-Schwarz inequality which yields:
\begin{align*}
 \sum_{i=1}^K  \sqrt{ \sum_{t=1}^T \frac{  \theta_{i} \indic\{i \in S_t\}}{\theta_0 + \Theta_{S_t}}  } \leq \sqrt{    K  \sum_{t=1}^T \frac{ \sum_{i=1}^K  \theta_{i} \indic\{i \in S_t\}}{\theta_0 + \Theta_{S_t}}  } 
 \leq \sqrt{KT} \,.
\end{align*}
Finally, combining the above result with~\eqref{eq:wtd1} and~\eqref{eq:wtd2} concludes the proof
\[
     \regu_T  \lesssim TP(\cE^c) + O(\tau_0) + \sqrt{\theta_{\max} x KT \log T}  \,.
\]
Choosing $x = 2 \log T$ ensures $TP(\cE^c) \leq O(1)$ and $\tau_0 \leq O(\log T)$. 
\end{proof}

       \section{Generalization: Improved dependance on $\theta_{\max}$ with Adaptive Pivot Selection} %
\label{sec:abi}

A problem with Algorithm~\ref{alg:wf} stems from estimating all $\theta_i$ based on pairwise comparisons with item $0$. When $\theta_{\max} \gg \theta_0 = 1$, item $0$ may not be sampled enough as the winner, leading to poor estimators. This deficiency contributes to the suboptimal dependence on $\theta_{\max}$ observed in Theorems~\ref{thm:topm_wf} and~\ref{thm:utilm_wf}. A similar issue arises in prior work, such as \cite{assort-mnl}, which remedies this assuming $\theta_0 = \theta_{\max} = 1$, which is however a very strong assumption.

We propose the following fix to optimize the pivot. 
For all $i,j \in [K]\cup \{0\}$ we define $\gamma_{ij} = \frac{\theta_i}{\theta_j}$, and the estimators:
\[
    \gamma_{ij,t}^{\ucb} = \frac{p_{ij,t}^{\ucb}}{(1-p_{ij,t}^{\ucb})_+} \quad  \text{ and } \quad \gamma_{ii,t}^{\ucb} = 1\,,
\]
where $p_{ij,t}^{\ucb}$ are defined in~\eqref{eq:pucb_def}.
Note that with this definition we have $\theta_t^{\ucb} = \gamma_{i0,t}^{\ucb}$. 
For all rounds $t$, the algorithm \aalgwf{} computes the upper-confidence-bounds:
\[
    \hat \theta_{i,t}^{\ucb} := \min_{j \in [K] \cup \{0\}} \gamma_{ij,t}^{\ucb} \gamma_{j0,t}^{\ucb} \,.
\]
The algorithm then selects
\[
    S_t = \argmax_{|S| \leq m} \cR(S, \hat{\theta}_t^{\ucb}) \,.
\]
We offer below a regret bound that underscores the value of optimizing the pivot when $\theta_{\max} \gg K$. 
In the experimental section (\cref{sec:expts}), we show that {this} version surpasses the performance of the state-of-the-art algorithm for AOA. \emph{It is important to note that while we present the algorithm and analysis specifically for the weighted objective with winner feedback, it can be readily adapted to other objectives. This adaptation involves replacing $\cR(S, \theta)$ with the new objective in the analysis, as long as \cref{lem:wtd_util} remains valid.}

\begin{theorem}
\label{thm:thetamaxucb}
Let $\theta_{\max} \geq 1$. For any instance of PL model on K items with parameter $\theta \in [0,\theta_{\max}]^K$ and weights $\r \in [0,1]^K$, the procedure described above has a regret upper-bounded as
\[
    \regu_T = O\big(\sqrt{\min\{\theta_{\max},K\} KT} \log T \big)
\]
as $T \to \infty$ for the choice $x = 2\log T$ (when definining $p_{ij,t}^{\ucb}$). 
\end{theorem}

\begin{rem}
It is noteworthy that asymptotically, when $\theta_{\max}$ is constant, the regret is $O(K\sqrt{T}\log T)$, eliminating any dependence on $\theta_{\max}$ in the asymptotic regime. This allows for handling scenarios where the No-Choice item is highly unlikely, which is not achievable in previous works such as \cite{assort-mnl}. The latter heavily relies on the regular selection of the No-Choice item to construct their estimators. 
\end{rem}

The proof is deferred to the appendix, with a key step relying on selecting the pivot  $j_t = \argmax_{j \in S_t \cup \{0\}} \theta_j$.  The use of $|\hat \theta_{i,t}^{\ucb} - \theta_i| \leq |\gamma_{ij_t,t}^{\ucb} - \theta_i|$ provides confidence upper-bounds with an improved dependence on $\theta_{\max}$ , leveraging the fact that $\theta_{j_t} \geq \theta_i$. Due to the varying pivot over time, a telescoping argument introduces an additive factor $\sqrt{K}$. 


        \section{Experiments}
\label{sec:expts}
We run extensive simulations to compare our methods with the state-of-the-art MNL-Bandit algorithm \cite{assort-mnl} and provide a comparative empirical evaluation based on different experimental setups and choice of PL parameters $\btheta$. The details of our empirical results are given below: 

\textbf{Algorithms.} 
We evaluate the performance of our main algorithm \textbf{(1). \aalgwf} (\cref{sec:abi}) with adaptive pivot, referred as \textbf{``Our Alg-1 (Adaptive Pivot)"}, with the following two algorithms:
\textbf{(2). \algwf\,} (\cref{sec:alg_wi}) referred as \textbf{``Our Alg-2 (No-Choice Pivot)"},  
\textbf{(3). MNL-UCB:} The proposed algorithm in \cite{assort-mnl} in Algorithm-$1$.

\textbf{Different PL $(\btheta)$ Environments.}
We report our experiment results on two $\pl$ datasets with $K = 50$ items: (1) Arith$50$ with \pl\ parameters $\theta_i = 1 - (i-1)0.2, ~\forall i \in [50]$.
(2) Bad$50$ with \pl\ parameters $\theta_i = 0.6, ~\forall i \in [50]\sm\{25\}$ and $\theta_{25} = 0.8$. For simplicity of computing the assortment choices $S_t$, we assume $r_i = 1, ~\forall i \in [K]$.

\begin{rem}
Note in our PL environments, we set $\theta_{\max} = 1$ (to avoid scaling of $\regt$\, with ~$\theta_{\max}$). Since the choice probabilities of PL models are scale-independent, our choice of parameters is equivalent to assuming another choice model $(\theta'_1,\theta'_2,\ldots,\theta'_K)$, $\theta'_0 = 1$ and adjusting $\theta'_i = \nicefrac{\theta_i}{\theta_0}, ~\forall i \in [K]$. 
\end{rem}

In all the experiments, we report the averaged performance of the algorithms across $100$ runs.

\textbf{(1). Averaged Regret with weak NC $(\theta_{\max}/\theta_0 \gg 1)$ (Fig. \ref{fig:reg_wi}):} In our first set of experiments we report the average performance of the above three algorithms against \topm-\aa\, Regret ($\regt_T$) and \utilm-\aa\, Regret ($\regu_T$). We set $m=5$ and $\theta_0 = 0.01$ for this set of experiments. 
\vspace{-2pt}
\begin{figure}[h]
\vspace{-5pt}
\hspace{-1pt}
\includegraphics[trim={0cm 0cm 0cm 0},clip,scale=0.1,width=0.26\textwidth]{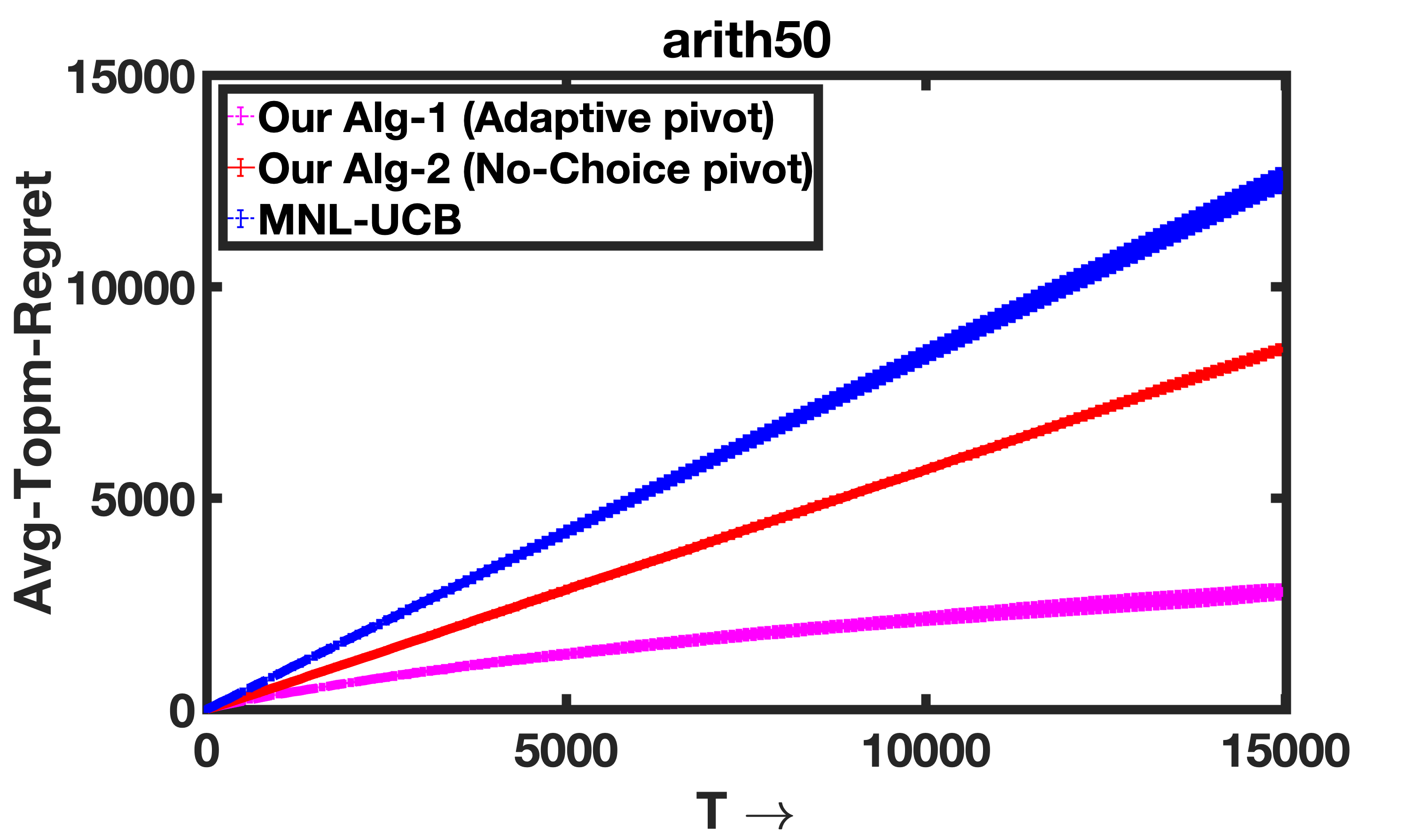}
\hspace{-12pt}
\includegraphics[trim={3.2cm 0cm 0cm 0},clip,scale=0.1,width=0.26\textwidth]{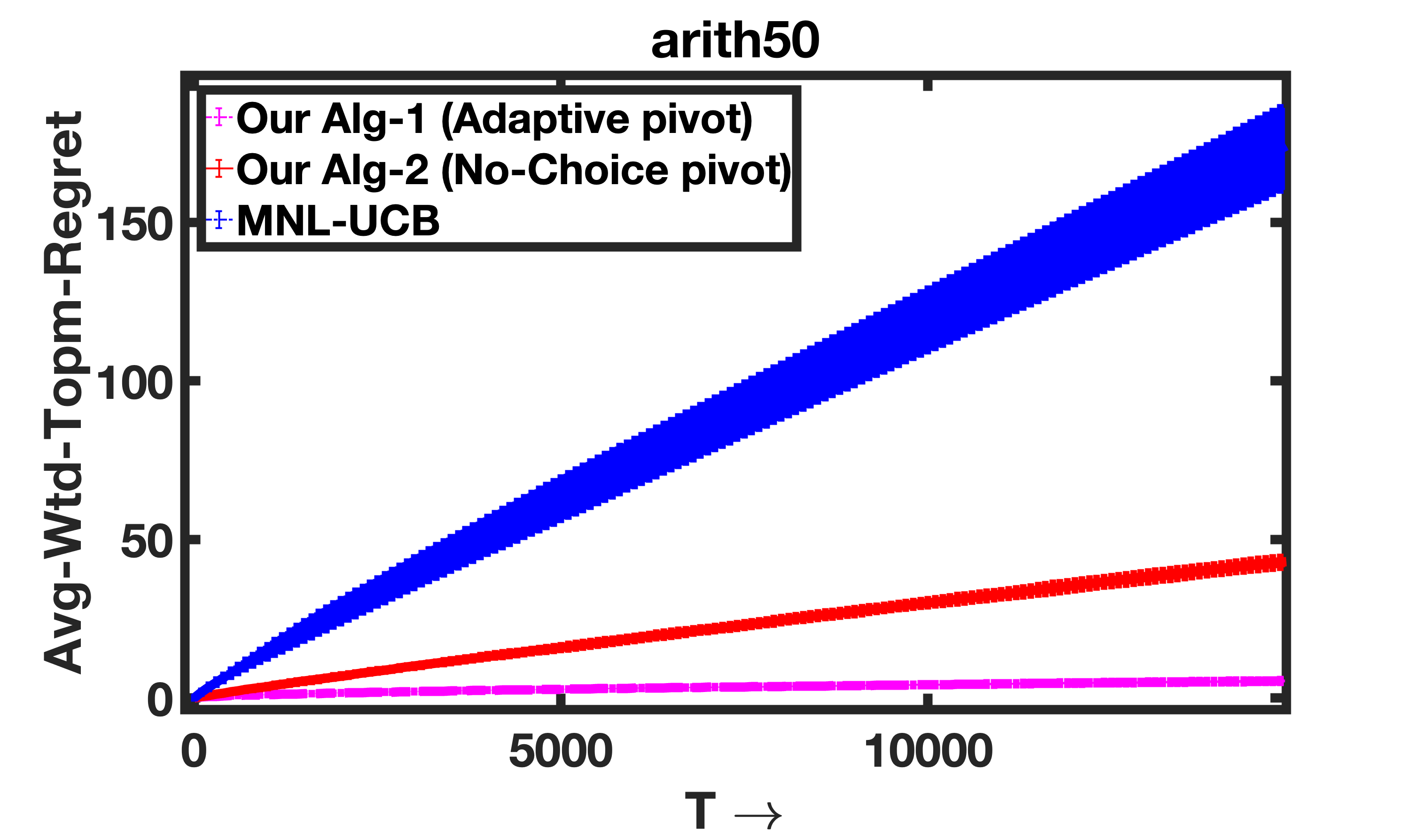}
\hspace{-12pt}
\includegraphics[trim={3.2cm 0cm 0cm 0},clip,scale=0.1,width=0.26\textwidth]{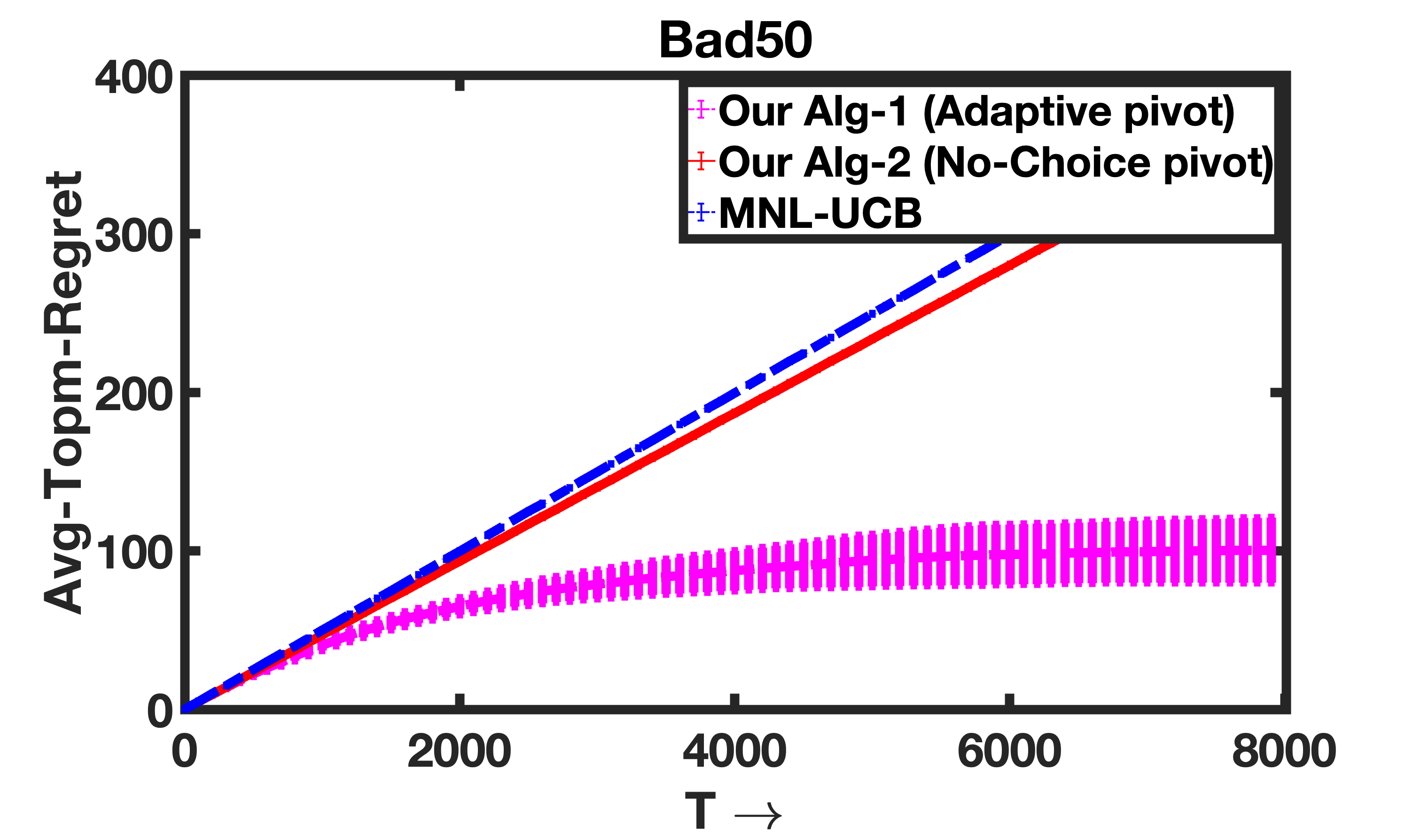}
\hspace{-12pt}
\includegraphics[trim={3.2cm 0cm 0cm 0},clip,scale=0.1,width=0.26\textwidth]{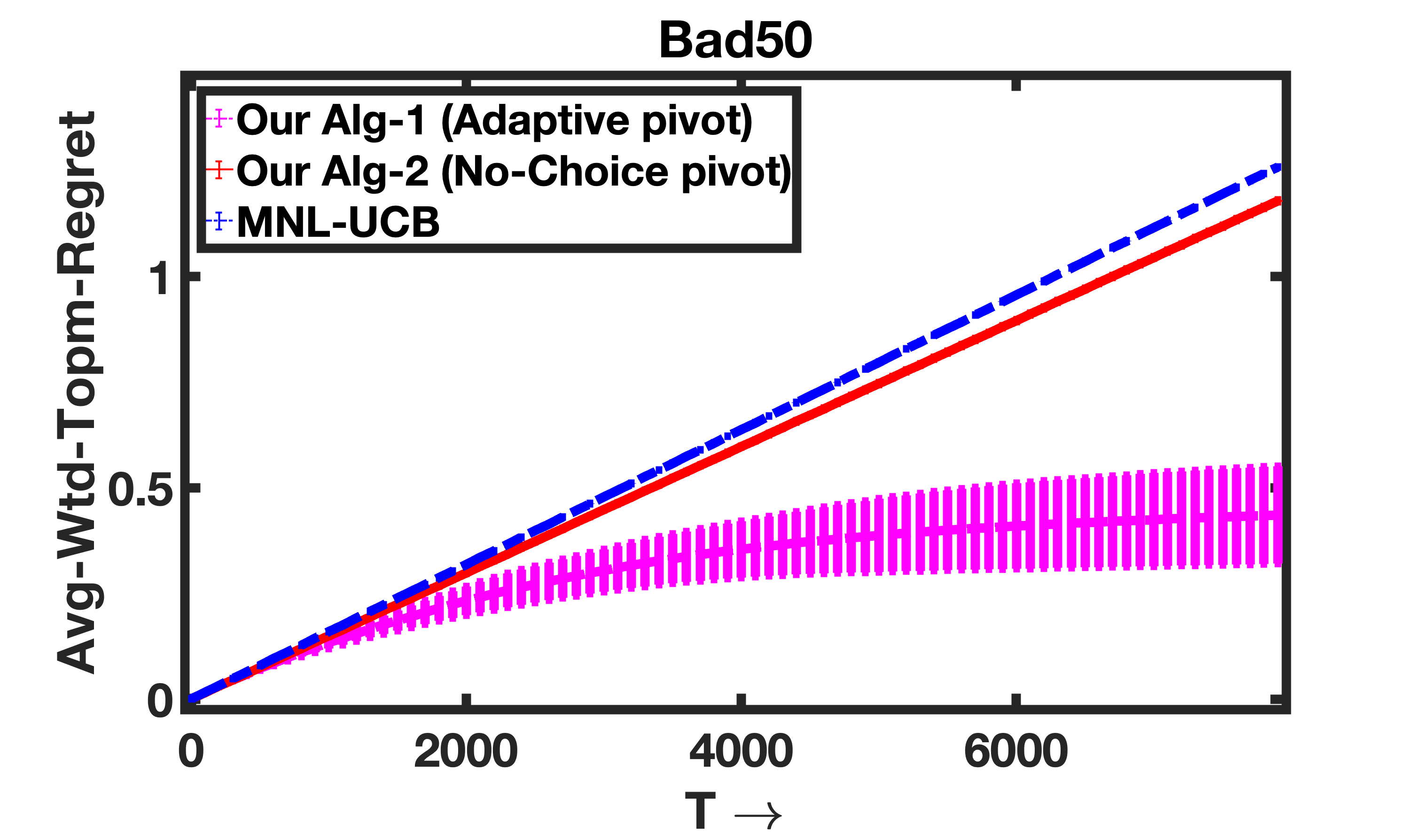}
\vspace{-15pt}
\caption{Averaged Regret for $m = 5$, $\theta_0 = 0.01$}
\label{fig:reg_wi}
\hspace{0pt}
\vspace{-15pt}
\end{figure}

\cref{fig:reg_wi} shows that our algorithm \aalgwf\, (with adaptive pivot) performs best, with a wide performance margin to the rest of the two algorithms, while our algorithm \algwf\, with no-choice (NC) pivot still outperforms MNL-UCB. But overall both of them perform poorly for $\theta_0 = 0.01$, as also expected since $\theta_{\max}$ is large compared to $\theta_0$ for these instances.

\textbf{(2). Averaged Regret vs No-Choice \pl\, Parameter $(\theta_0)$ (Fig. \ref{fig:reg_nc}):}
In this experiment, we evaluate the regret performance of our algorithm \aalgwf. We report the experiment on Artith$50$ \pl\ dataset and set the subsetsize $m = 6$, $\theta_0 = \{1, 0.5,0.1,0.05,0.01,0.005,0.001\}$. \cref{fig:reg_wi} shows, the performance gap between our algorithm \aalgwf\, (with adaptive pivot) increases with decreasing $\theta_0$ -- this is expected since the estimates of MNL-UCB strongly depends on the fact that $\theta_0$ needs to be maximum of $\{\theta_i\}_{i \in [K]}$ without which the algorithm performs increasingly suboptimaly due to poor estimation of the \pl\ parameters. 

\vspace{-2pt}
\begin{figure}[h]
\begin{center}
\vspace{-5pt}
\hspace{0pt}
\includegraphics[trim={0.4cm 0cm 0cm 0},clip,scale=0.1,width=0.27\textwidth]{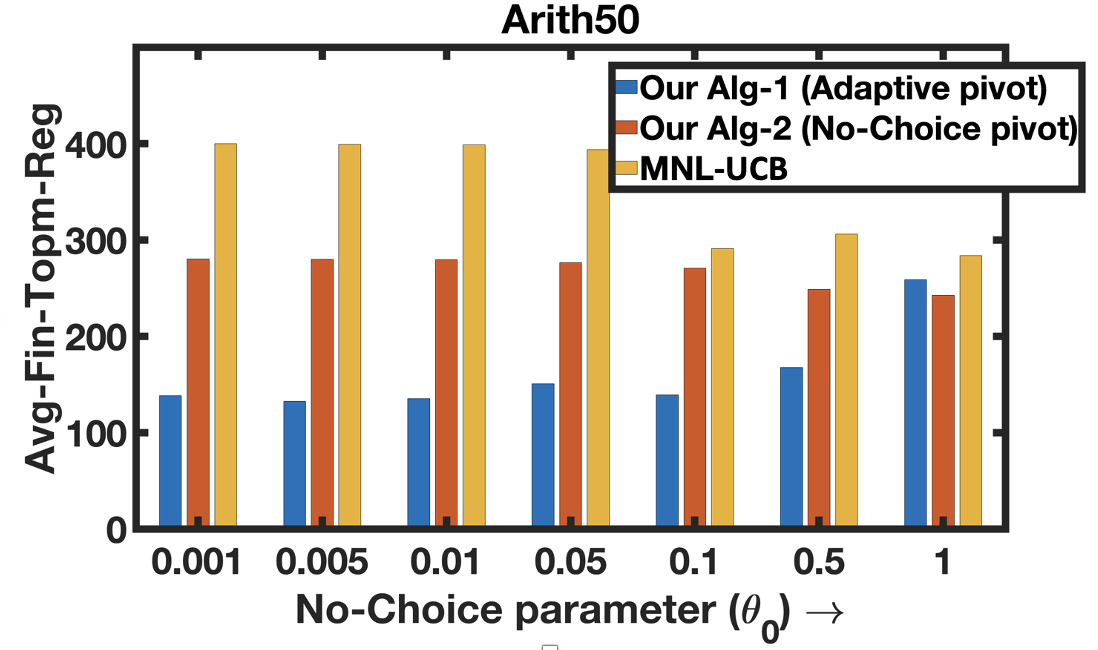}
\includegraphics[trim={3.2cm 0cm 0cm 0},clip,scale=0.1,width=0.27\textwidth]{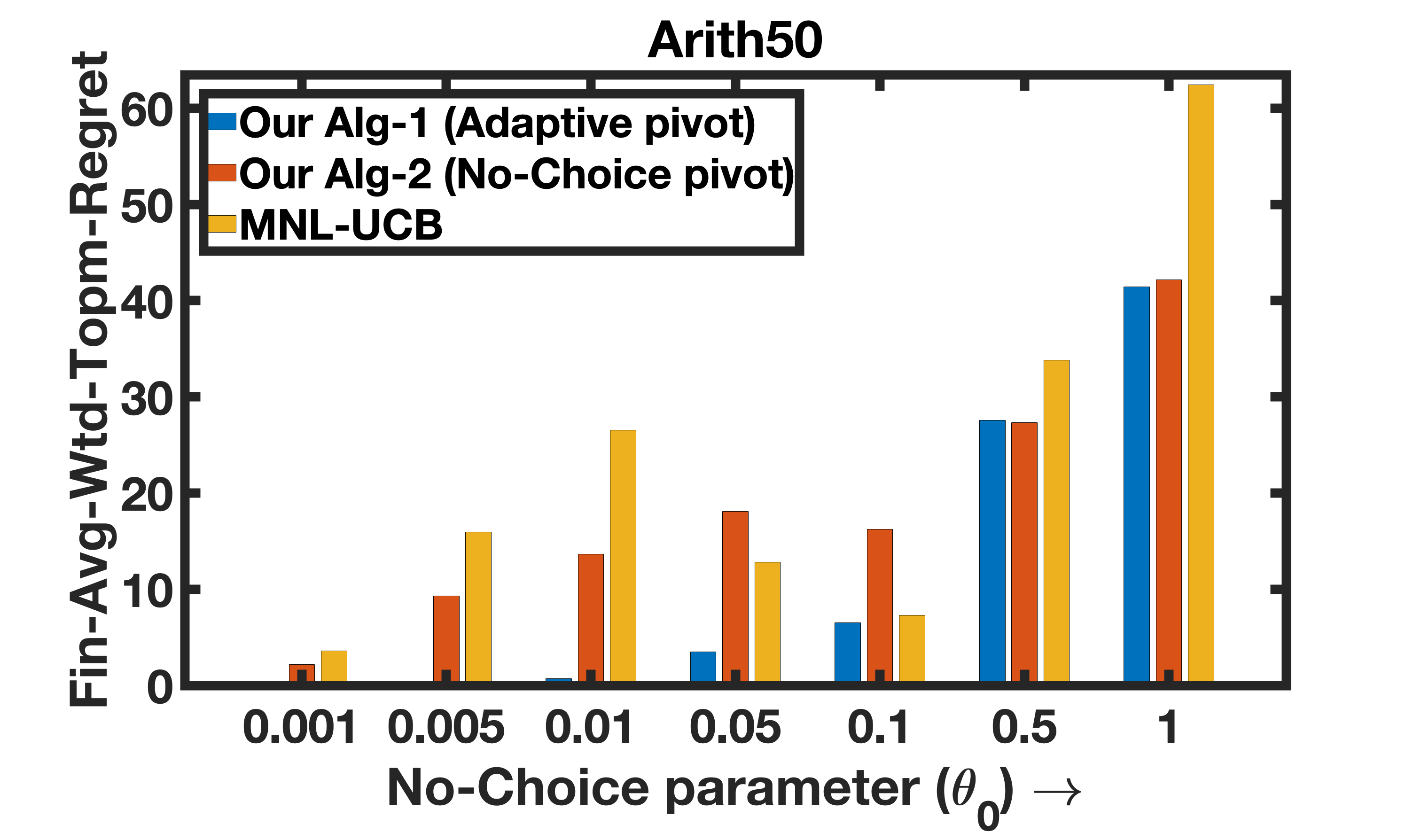}
\end{center}
\vspace{-10pt}
\caption{Comparative performances of the algorithms for varying strength of NC $(\theta_0)$ for $m = 5$}
\label{fig:reg_nc}
\hspace{0pt}
\vspace{-15pt}
\end{figure}
\vspace{-0pt}


\textbf{(3). Averaged Regret vs Length of the rank-ordered feedback ($k$) (Fig. \ref{fig:reg_tf}):}
We also run a thought experiment to understand the tradeoff between learning rate with $k$-length rank-ordered feedback, where given any assortment $S_t \subseteq [K]$ of size $m$, the learner gets to see the top-$k$ draws ($k\leq m$) from the PL model without replacement. Clearly, this is a stronger feedback than the winner (i.e. top-$1$ for $k=1$) feedback and one would expect a faster learning rate with increasing $k$ as the learner gets to observe more information from the PL model each time. We report the experiment on Artith$50$ \pl\ dataset and set the subsetsize $m = 30$ with $k = 1, 2, 4, 8$. As expected, we see a multiplicative reduction on both \topm-\aa\, Regret ($\regt_T$) and \utilm-\aa\, Regret ($\regu_T$) with increasing $k$, as reflected in \cref{fig:reg_tf}.

\vspace{-2pt}
\begin{figure}[h]
\begin{center}
\vspace{-5pt}
\hspace{0pt}
\includegraphics[trim={3.2cm 0cm 0cm 0},clip,scale=0.1,width=0.27\textwidth]{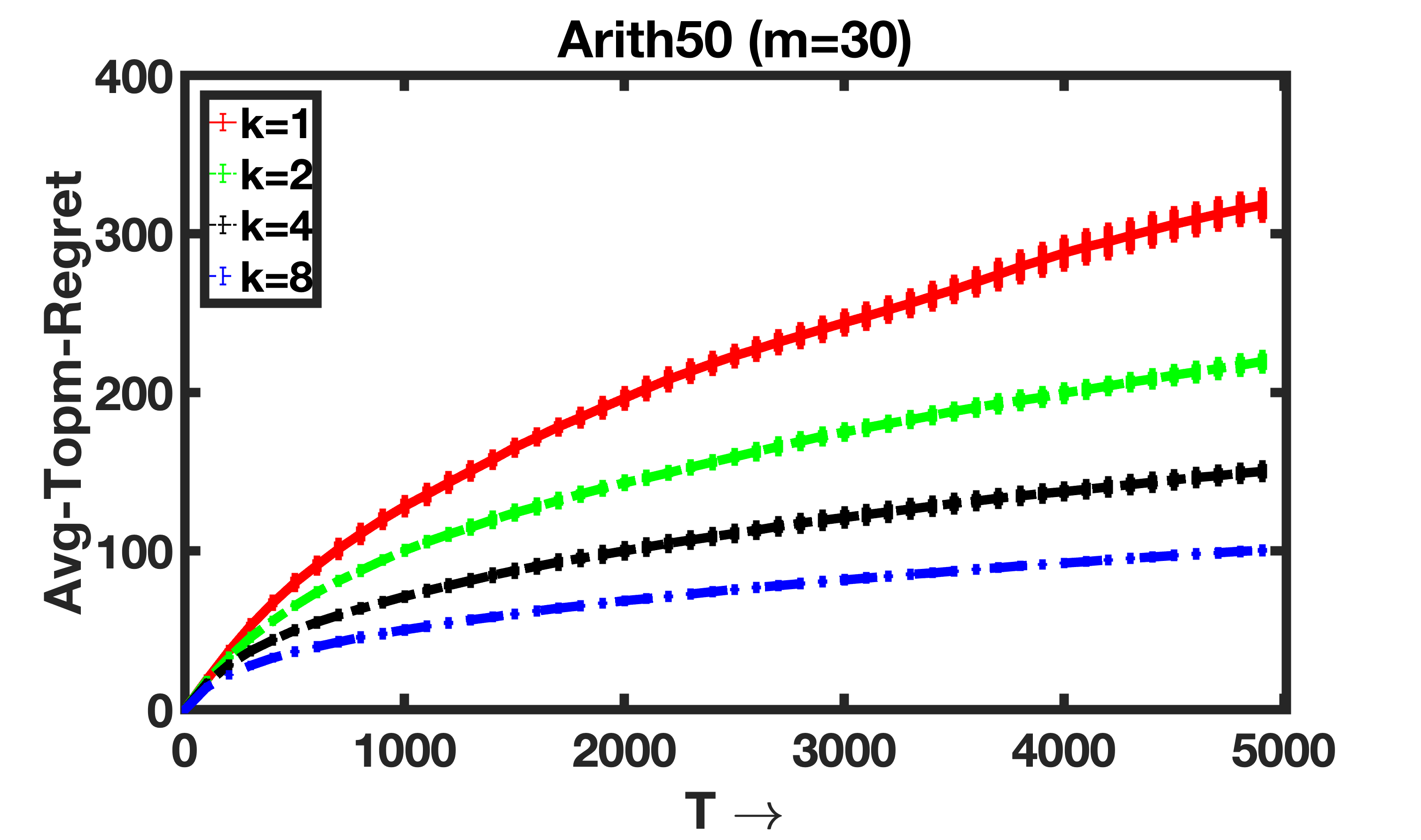}
\includegraphics[trim={3.2cm 0cm 0cm 0},clip,scale=0.1,width=0.27\textwidth]{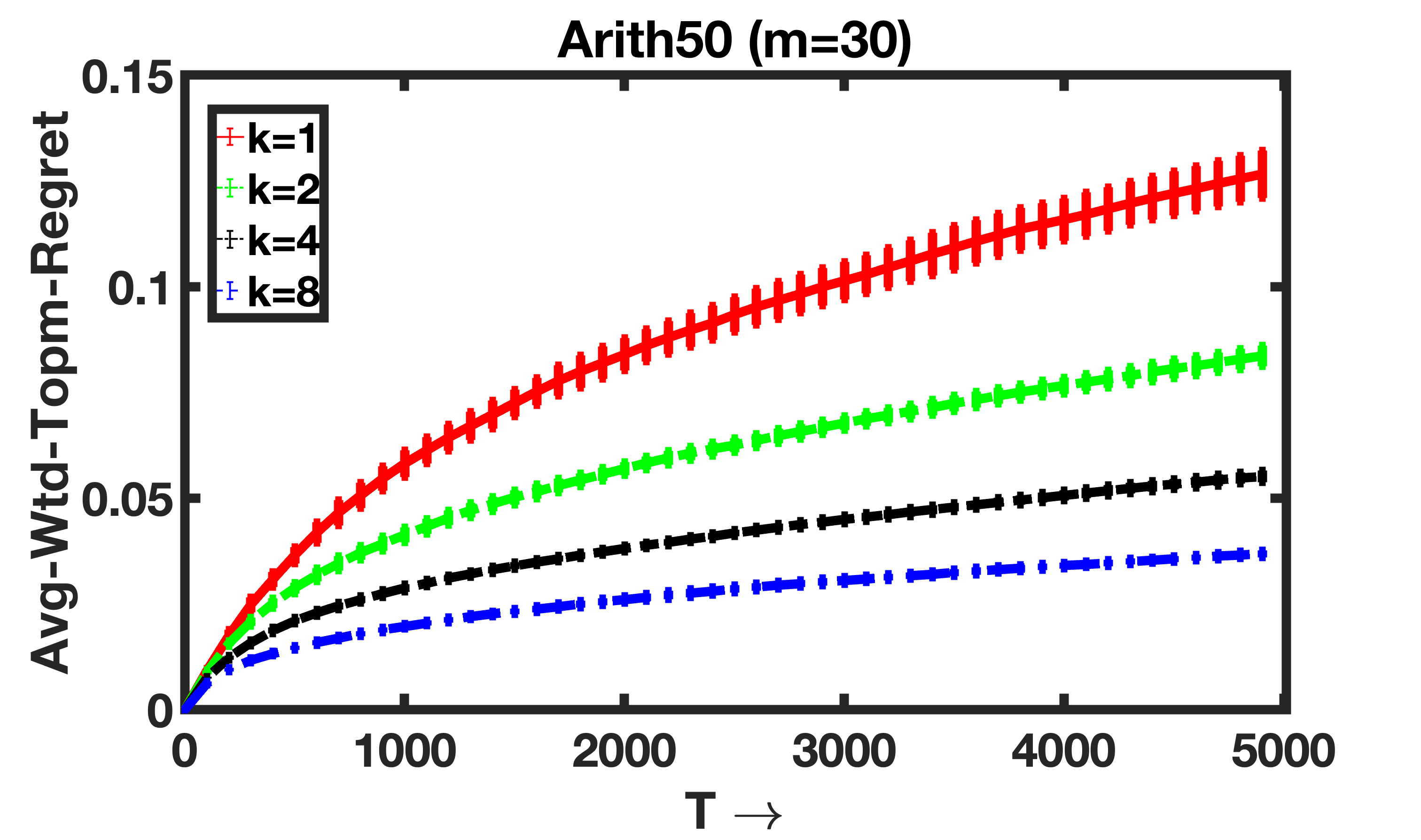}
\end{center}
\vspace{-10pt}
\caption{Tradofff: Averaged Regret vs length of the $k$ rank-ordered feedback for $k = 1,2,4,8$, $m = 30$}
\label{fig:reg_tf}
\hspace{0pt}
\vspace{-15pt}
\end{figure}

        \vspace{-10pt}
\section{Conclusion}
\label{sec:conclusion}
\vspace{-5pt}
%


We address the Active Optimal Assortment Selection problem with PL choice models, introducing a versatile framework (\textit{\aa}) that eliminates the need for a strong default item, typically assumed as the No-Choice (NC) item in the existing literature. Our proposed algorithms employ a novel 'Rank-Breaking' technique to establish tight concentration guarantees for estimating the score parameters of the PL model. Our algorithmic approach stands out for its practicality and avoids the suboptimal practice of repeatedly selecting the same set of items until the default item prevails. This is particularly beneficial when the default item's quality $(\theta_0)$ is significantly lower than the quality of the best item $(\theta_{\max})$. Our algorithms are efficient to implement, provably optimal (up to log factors), and free from unrealistic assumptions on the default item. In summary, this work pioneers an efficient and optimal assortment selection approach for PL choice models without imposing restrictive conditions.

\textbf{Future Works. } Among many interesting questions to address in the future, it will be interesting to understand the role of the No-Choice (NC) item in the algorithm design, precisely, can we design efficient algorithms without the existence of NC items with a regret rate still linear in $\theta_{\max}$? Further, it will be interesting to extend our results to more general choice models beyond the PL model \cite{assort-nested,assort-mallows,assort-markov}. What is the tradeoff between the subsetsize $m$ and the regret for such general choice models? Extending our results to large (potentially infinite) decision spaces and contextual settings would also be a very useful and practical contribution to the literature of assortment optimization. 


	\bibliographystyle{plainnat}
	\bibliography{refs,bib_bdb,assort}
	
	\newpage
\appendix
\onecolumn{
\section*{\centering\large{Supplementary: \papertitle}}
\vspace*{1cm}

\section{Preliminaries: Some Useful Concepts for PL choice models}
\label{sec:prelims}

\subsection{Plackett-Luce (PL): A Discrete Choice Model}
\label{sec:RUM}  

A discrete choice model specifies the relative preferences of two or more discrete alternatives in a given set. %
%
A widely studied class of discrete choice models is the class of {\it Random Utility Models} (RUMs), which assume a ground-truth utility score $\theta_{i} \in \R$ for each alternative $i \in [n]$, and assign a conditional distribution $\cD_i(\cdot|\theta_{i})$ for scoring item $i$. To model a winning alternative given any set $S \subseteq [n]$, one first draws a random utility score $X_{i} \sim \cD_i(\cdot|\theta_{i})$ for each alternative in $S$, and selects an item with the highest random score. 

%

One widely used RUM is the {\it Multinomial-Logit (MNL)} or {\it Plackett-Luce model (PL)}, where the $\cD_i$s are taken to be independent Gumbel distributions with parameters $\theta'_i$ \citep{Az+12}, i.e., with probability densities $$
\cD_i(x_{i}|\theta'_i) = e^{-(x_j - \theta'_j)}e^{-e^{-(x_j - \theta'_j)}}, \qquad \theta'_i \in R, ~ \forall i \in [n] \,.
$$  
Moreover assuming $\theta'_i = \ln \theta_i$, $\theta_i > 0 ~\forall i \in [n]$, it can be shown in this case the probability that an alternative $i$ emerges as the winner in the set $S \ni i$ becomes: 
$
\P(i|S) = \frac{{\theta_i}}{\sum_{j \in S}{\theta_j}}.
$

Other families of discrete choice models can be obtained by imposing different probability distributions over the utility scores $X_i$, e.g. if $(X_1,\ldots X_n) \sim \cN(\btheta,\boldsymbol \Lambda)$ are jointly normal with mean $\btheta = (\theta_1,\ldots \theta_n)$ and covariance $\boldsymbol \Lambda \in \R^{n \times n}$, then the corresponding RUM-based choice model reduces to the {\it Multinomial Probit (MNP)}. 
 

\subsection{Rank Breaking} 
\label{sec:rb}

\rmd\,\emph{Rank breaking} (RB) is a well-understood idea involving the extraction of pairwise comparisons from (partial) ranking data, and then building pairwise estimators on the obtained pairs by treating each comparison independently \citep{KhetanOh16,SueIcml+17}, e.g., a winner $a$ sampled from among ${a,b,c}$ is rank-broken into the pairwise preferences $a \succ b$, $a \succ c$. We use this idea to devise estimators for the pairwise win probabilities $p_{ij} = \P(i|\{i,j\}) = \theta_i/(\theta_i + \theta_j)$ for our problem setting. We used the idea of RB in both our algorithms (\algwf\, and \aalgwf) to update the pairwise win-count estimates $w_{i,j,t}$ for all the item pairs $(i,j) \in [K] \times [K]$, which is further used for deriving the empirical pairwise preference estimates $\hp_{ij,t}$, at any time $t$.  



\subsection{Parameter Estimation with PL based preference data}
\label{sec:est_pl_score}

\begin{restatable}[Pairwise win-probability estimates for the PL model \citep{SGrank18}]
{lemma}{plsimulator}
\label{lem:pl_simulator}
Consider a Plackett-Luce choice model with parameters $\btheta = (\theta_1,\theta_2, \ldots, \theta_n)$, and fix two  items $i,j \in [n]$. Let $S_1, \ldots, S_T$ be a sequence of (possibly random) subsets of $[n]$ of size at least $2$, where $T$ is a positive integer, and $i_1, \ldots, i_T$ a sequence of random items with each $i_t \in S_t$, $1 \leq t \leq T$, such that for each $1 \leq t \leq T$, (a) $S_t$ depends only on $S_1, \ldots, S_{t-1}$, and (b) $i_t$ is distributed as the Plackett-Luce winner of the subset $S_t$, given $S_1, i_1, \ldots, S_{t-1}, i_{t-1}$ and $S_t$, and (c) $\forall t: \{i,j\} \subseteq S_t$ with probability $1$. Let $n_i(T) = \sum_{t=1}^T \P(i_t = i)$ and $n_{ij}(T) = \sum_{t=1}^T \P(\{i_t \in \{i,j\}\})$. Then, for any positive integer $v$, and $\eta \in (0,1)$,
\vspace{-3pt}
\begin{align*}
& \P\left( \frac{n_i(T)}{n_{ij}(T)} - \frac{\theta_i}{\theta_i + \theta_j} \ge \eta, \; n_{ij}(T) \geq v \right) \leq e^{-2v\eta^2},\\
& \P\left( \frac{n_i(T)}{n_{ij}(T)} - \frac{\theta_i}{\theta_i + \theta_j} \le -\eta, \; n_{ij}(T) \geq v \right) \leq e^{-2v\eta^2}. 
\end{align*}
\end{restatable}

\section{Omitted proofs}
\label{app:algo_db}

\subsection{Proof of Lemma~\ref{lem:pconc}}

\begin{proof}[Proof of Lemma~\ref{lem:pconc}]
    Let $T\geq 1$, $x >0$ and $i,j\in [K]$. Applying Thm. 1 of \cite{audibert2009exploration}, with probability at least $1-\beta(x,T)$, we get simultaneously for all $t \in [T]$, 
    \begin{equation}
        \big| \hat p_{ij,t} - p_{ij} \big| \leq \sqrt{\frac{2 \hat p_{ij,t}(1-\hat p_{ij,t})x}{n_{ij,t}}} + \frac{3x}{n_{ij,t}} \,,
        \label{eq:ucb_p}
    \end{equation}
    where $\beta(x,T) = 3\inf_{1 < \alpha \leq 3}\min \big\{ \frac{\log T}{\log \alpha},T\big\}e^{-x/\alpha} \leq 3T e^{-x}$. Note that the inequality holds true although $n_{ij,t}$ is a random variable. This, shows the first inequality
    \[
        p_{ij} \leq p_{ij,t}^{\ucb} \,.
    \]
    For the second inequality, \eqref{eq:ucb_p} implies
    \begin{align}
        p_{ij,t}^{\ucb} 
            & = \hat p_{ij,t} + \sqrt{\frac{2 \hat p_{ij,t}(1-\hat p_{ij,t}) x}{n_{ij,t}}} + \frac{3x}{n_{ij,t}} \nonumber \\
            & \leq  p_{ij} + 2\sqrt{\frac{2 \hat p_{ij,t}(1-\hat p_{ij,t})x}{n_{ij,t}}} + \frac{6x}{n_{ij,t}} \,.\label{eq:ucb_var}
    \end{align}
    Furthermore, because $x \mapsto x(1-x)$ is 1-Lipschitz on $[0,1]$, we have
    \begin{align*}
        \big|\hat p_{ij,t}(1-\hat p_{ij,t}) & - p_{ij}(1-p_{ij})\big| \leq  \big| \hat p_{ij,t} - p_{ij} \big| \\
            & \stackrel{\eqref{eq:ucb_p}}{\leq} \sqrt{\frac{2 \hat p_{ij,t}(1-\hat p_{ij,t})x}{n_{ij,t}}} + \frac{3x}{n_{ij,t}} \,.
    \end{align*}
    Therefore,
    \begin{align*}
        \hat p_{ij,t}(1-\hat p_{ij,t}) 
            & \leq p_{ij}(1-p_{ij}) + \sqrt{\frac{2 \hat p_{ij,t}(1-\hat p_{ij,t})x}{n_{ij,t}}} + \frac{3x}{n_{ij,t}} \\
            & \leq \bigg(\sqrt{p_{ij}(1-p_{ij})} + \sqrt{\frac{3x}{n_{ij,t}}}\bigg)^2 \,,
    \end{align*}
    which yields
    \begin{equation}
        \sqrt{\hat p_{ij,t}(1-\hat p_{ij,t}}) \leq \sqrt{p_{ij}(1-p_{ij})} + \sqrt{\frac{3x}{n_{ij,t}}}\,.\label{eq:ucb_std}
    \end{equation}
    Plugging back into~\eqref{eq:ucb_var}, we get
    \[
        p_{ij,t}^{\ucb} \leq 2\sqrt{\frac{2  p_{ij}(1- p_{ij})x}{n_{ij,t}}} + \frac{11x}{n_{ij,t}}\,.
    \]
\end{proof}

\subsection{Proof of Lemma~\ref{lem:tconc}}

\begin{proof}
Let $i \in [K]$ and $x>0$. Then, by a union bound on Lemma~\ref{lem:pconc} and~\ref{lem:nibound}, with probability at least $1 - 4 T e^{-x}$,~\eqref{eq:ucb_p1} and~\eqref{eq:nibound} hold true for all $t \in [T]$. We consider this high-probability event in the rest of the proof. Define the function $f:x\mapsto x/(1-x)_+$ on $[0,1]$ (with the convention $f(1) = +\infty$), so that $\theta_{i,t}^{\ucb} = f(p_{i0,t}^{\ucb})$ and $\theta_i = f(p_{i0})$. Because $f$ is non-decreasing, and $p_{i0,t}^{\ucb} \geq p_{i0}$ by~\eqref{eq:ucb_p1}, we have
\begin{equation}
    \theta_{i,t}^{\ucb} \geq \theta_i \,. \label{eq:theta_ucb}
\end{equation}
Furthermore, denote 
\begin{equation}
    \Delta_{i,t}  :=  2\sqrt{\frac{2 p_{ij}(1-p_{ij}) x}{n_{i0,t}}} + \frac{11x}{n_{i0,t}} 
         = 2\sqrt{\frac{2 \theta_0 \theta_i x}{(\theta_0+\theta_i)^2 n_{i0,t}}} + \frac{11x}{n_{i0,t}} \,. \label{eq:deltait}
\end{equation}
In the rest of the proof we assume, $n_{i0,t} \geq 69 x (\theta_0 +\theta_i)$. Then, using that $\theta_0 \theta_i \leq \theta_0 + \theta_i$ since $\theta_0 = 1$, it implies
\begin{equation}
    (\theta_0+\theta_i)  \Delta_{i,t}  \leq 2 \sqrt{\frac{2\theta_0 \theta_i x}{n_{i0,t}}} + \frac{11x(\theta_0 + \theta_i)}{n_{i0,t}} \leq  \frac{1}{2} \nonumber\,,
\end{equation}
and
\begin{align*}
    p_{i0} + \Delta_{i,t}  = \frac{\theta_i}{\theta_0 + \theta_i} + \Delta_{i,t} \leq \frac{\theta_i + 1/2}{\theta_i + 1} < 1. 
\end{align*}
Thus, because $f$ is non-decreasing
\begin{align*}
    \theta_{i,t}^{\ucb} - \theta_i
        & = f(p_{i0,t}^{\ucb}) - f(p_{i0})  \\
        & \stackrel{\eqref{eq:ucb_p1}}{\leq} f\big(p_{i0} +  \Delta_{i,t}\big) - f(p_{i0}) \\
        & = \frac{p_{i0} +  \Delta_{i,t}}{1-p_{i0} -  \Delta_{i,t}} - \frac{p_{i0}}{1-p_{i0}} \\
        & = \frac{ \Delta_{i,t}}{(1-p_{i0})(1-p_{i0}- \Delta_{i,t})}\\
        & = \frac{(\theta_0 + \theta_i)^2  \Delta_{i,t}}{1- (\theta_0 + \theta_i) \Delta_{i,t}}\\
        & \leq  2 (\theta_0 + \theta_i)^2  \Delta_{i,t} \\
        & \stackrel{\eqref{eq:deltait}}{\leq} 4(\theta_0+\theta_i)\sqrt{\frac{2 \theta_0 \theta_i x}{n_{i0,t}}} + \frac{22x(\theta_0+\theta_i)^2}{n_{i0,t}}  \,,\\
\end{align*}
which concludes the proof. 
\end{proof}

\subsection{Proof of Lemma~\ref{lem:nibound}}

\begin{proof}
Let $T\geq 1$ and $i \in [K]$. Recall that $\tau_{i,t} = \sum_{s=1}^{t-1} \indic\{i_s  \in S_s\}$ is the number of times $i$ was played at the start of round $t$ and $n_{i0,t} = \sum_{s=1}^{t-1} \indic\{i_t \in \{i,0\}, i \in S_t\}$ is the number of times $i$ or $0$ won up to round $t$ when played together. When $i$ is played the probability of $0$ or $i$ to win is 
\[
   \P(i_t \in \{i,0\}| S_t) =  \frac{\theta_0 + \theta_i}{\theta_0 + \Theta_{S_t}} \geq  \frac{\theta_0 + \theta_i}{\theta_0 + \Theta_{S^*}} \,. 
\]
Therefore, applying Chernoff-Hoeffding inequality together with a union bound (to deal with the fact that $\tau_{i,t}$ is random), we have with probability at least $1-T e^{-x}$
\[
    n_{i0,t} \geq \frac{\theta_0 + \theta_{i}}{\theta_0 + \Theta_{S^*}} \tau_{i,t} - \sqrt{\frac{\tau_{i,t} x}{2}}
\]
simultaneously for all $t \in [T]$. Noting that 
\[
\frac{\theta_0 + \theta_{i}}{\theta_0 +  \Theta_{S^*}} \tau_{i,t} - \sqrt{\frac{\tau_{i,t} x}{2}} \geq \frac{\theta_0 + \theta_{i}}{2(\theta_0 + \Theta_{S^*})} \tau_{i,t}
\]
if $\tau_{i,t} \geq 2 x(\theta_0 + \Theta_{S^*})^2  \geq \frac{2 x (\theta_0 + \Theta_{S^*})^2}{(\theta_0 + \theta_i)^2} $ concludes the proof. 
\end{proof}

\subsection{Proof of Theorem \ref{thm:topm_wf}}
\label{app:thm_wf}

\begin{proof}
Let us define for any $S \subseteq [K]$,
\[
\Theta_S = \sum_{i \in S}\theta_i, ~~\text{ and }~~ \Theta_S^{\ucb}:=  \sum_{i \in S}\theta_i^{\ucb}.
\]

Let $\cE$ be the high-probabality event such that both Lemma~\ref{lem:tconc} and~\ref{lem:nibound} holds true. Then, $\P(\cE) \geq 1-4TKe^{-x}$. Let us first assume that $\cE$ holds true. Then, by Lemma~\ref{lem:tconc},
\begin{align*}
\regt_T & = \frac{1}{m} \sum_{t=1}^T \Theta_{S^*} - \Theta_{S_t}  \\
    & \leq  \frac{1}{m} \sum_{t=1}^T \min\Big\{\Theta_{S^*}, \Theta_{S_t}^{\ucb} - \Theta_{S_t} \Big\} \hspace{2cm} \leftarrow \text{ because } \Theta_{S^*} \leq \Theta_{S^*}^{\ucb} \leq \Theta_{S_t}^{\ucb} \text{ under the event $\cE$} \\
    & = \frac{1}{m} \sum_{t=1}^T \min\Big\{\Theta_{S^*}, \sum_{i \in S_t} \theta_{i,t}^{\ucb} - \theta_{i} \Big\}\\
    & \leq \frac{1}{m} \Theta_{S^*} \sum_{i=1}^K \bar \tau_{i0} + \frac{1}{m} \sum_{t=1}^T \sum_{i \in S_t} (\theta_{i,t}^{\ucb} - \theta_{i}) \indic\big\{ \tau_{i,t} \geq \bar \tau_{i0} \big\}  
\end{align*}
where $\bar \tau_{i0} = 2x (\theta_0 +\Theta_{S^*}) \max\{\theta_0 + \Theta_{S^*}, 69\} \leq 138 x (m+1)^2 \theta_{\max}^2$, where  $\theta_{\max} := \max_i \theta_i$.
Then, noting that if $\cE$ holds true, by Lemma~\ref{lem:nibound}, we also have $n_{i0,t} \geq \frac{1}{2(\theta_0 +\Theta_{S^*})} (\theta_0 + \theta_i) \tau_{i,t}$, which yields
\[
    \indic\{ \tau_{i,t} \geq \bar \tau_{i0} \} \leq \indic\{ n_{i0,t} \geq 69 x (\theta_0 + \theta_i) \}.
\]
Therefore, we can apply Lemma~\ref{lem:tconc} that entails,
\begin{align*}
\frac{1}{m} \sum_{t=1}^T &  \sum_{i \in S_t} (\theta_{i,t}^{\ucb} - \theta_{i}) \indic\big\{ \tau_{i,t} \geq \bar \tau_{i0} \big\} \\
     & \stackrel{\text{Lem.~\ref{lem:tconc}}}{\leq}    \frac{1}{m} \sum_{t=1}^T \sum_{i \in S_t}\Big( 4(\theta_0+\theta_i)\sqrt{\frac{2 \theta_0 \theta_i x}{n_{i0,t}}} + \frac{22x(\theta_0+\theta_i)^2}{n_{i0,t}}\Big) \indic\big\{  n_{i0,t} \geq 69 x (\theta_0 + \theta_i) \big\}  \\
    & \stackrel{\text{Lem~\ref{lem:nibound}}}{\leq}   \frac{1}{m} \sum_{t=1}^T \sum_{i \in S_t}\bigg(   8 \sqrt{\frac{ (\theta_0 +\Theta_{S^*})(\theta_0+\theta_i) \theta_0 \theta_i x}{\tau_{i,t}}} + \frac{44x (\theta_0 +\Theta_{S^*})(\theta_0+\theta_i)}{\tau_{i,t}} \bigg)  \\
    & \leq  \frac{1}{m} \sum_{i=1}^K  16 \sqrt{(\theta_0 + \Theta_{S^*})(\theta_0+\theta_i) \theta_0 \theta_i x \tau_{i,T}} + 44x (\theta_0 + \Theta_{S^*}) \sum_{i=1}^K (\theta_0+\theta_i)(1+\log(\tau_{i,T}))\,,
\end{align*}
where we used $\sum_{i=1}^n 1/\sqrt{i} \leq 2 \sqrt{n}$ and $\sum_{i=1}^n i^{-1} \leq 1+ \log n$. We thus have
\begin{align*}
\regt_T & \leq  138 x (m+1)^2 K \theta_{\max}^3 +  \frac{1}{m} \sum_{i=1}^K  16 \theta_{\max}^{3/2} \sqrt{ (m+1)  x \tau_{i,T}} + 44x (m+1)  (1 + \theta_{\max})^2 \sum_{i=1}^K(1+\log(\tau_{i,T})) \\
    & \leq 138 x (m+1)^2 K \theta_{\max}^3 +    16 \theta_{\max}^{3/2} \sqrt{2 x KT } + 88 x (m+1) K   \theta_{\max}^2 \Big(1+\log\Big(\frac{mT}{K}\Big)\Big) \,.
\end{align*}
Therefore,
\[
\E[\regt_T] \leq 12 \sqrt{2} x mK \theta_{\max}^3 +    16\theta_{\max}^{3/2} \sqrt{2 x KT } + 88 x m K  \theta_{\max}^2 \Big(1+\log\Big(\frac{mT}{K}\Big)\Big) + 4m KT^2e^{-x} \theta_{\max} 
\,.
\]
Choosing $x = 2 \log T$ concludes the proof. 
\end{proof}

\subsection{Proof of Theorem~\ref{thm:utilm_wf}}

\begin{proof}

Let $\cE$ be the high-probabality event such that Lemma~\ref{lem:tconc} and~\ref{lem:nibound} are satisfied, so that $\P(\cE) \geq 1-4KTe^{-x}$. Then, denoting $x \wedge y := \min\{x,y\}$,
\begin{align}
    \regu_T  & = \sum_{t=1}^T \E\big[\cR(S^*, \theta) - \cR(S_t, \theta)\big] 
    = \sum_{t=1}^T \E\big[(\cR(S^*, \theta) - \cR(S_t, \theta)) \indic\{\cE\} + (\cR(S^*, \theta) - \cR(S_t, \theta)) \indic\{\cE^c\}\big]  \nonumber \\
    & \leq \sum_{t = 1}^T \E\Big[ \big( (\cR(S_t,\theta_t^{\ucb}) - \cR(S_t, \theta)) \wedge \cR(S^*, \theta) \big) \indic\{\cE\} + \cR(S^*, \theta) \indic\{\cE^c\} \Big] \nonumber
\end{align}
because $\cR(S_t,\theta_t^{\ucb}) \geq \cR(S^*,\theta_{t}^{\ucb}) \geq \cR(S^*,\theta)$ under the event $\cE$ by Lemma~\ref{lem:wtd_util}. Then, using $\cR(S^*, \theta) \leq \max_i r_i \leq 1$, we get
\begin{equation}
    \regu_T   \leq \sum_{t = 1}^T \E\Big[ \big( (\cR(S_t,\theta_t^{\ucb}) - \cR(S_t, \theta)) \wedge 1 \big) \indic\{\cE\} + \indic\{\cE^c\} \Big]  \leq  4T^2Ke^{-x} + \sum_{t = 1}^T \E\Big[ \Big( \big( \cR(S_t,\theta_t^{\ucb}) - \cR(S_t, \theta) \big) \wedge 1\Big) \indic\{\cE\} \Big] \label{eq:regMNL1} \,.
\end{equation}
Let us upper-bound the second term of the right-hand-side
\begin{align}
 \sum_{t = 1}^T&  \E\Big[  \Big(\big( \cR(S_t,\theta_t^{\ucb}) - \cR(S_t, \theta) \big) \wedge 1\Big) \indic\{\cE\} \Big] 
     = \sum_{t=1}^T  \E\bigg[   \bigg( \bigg( \sum_{i\in S_t} \frac{r_i \theta_{i,t}^{\ucb} }{\theta_0 + \Theta_{S_t,t}^\ucb} - \frac{r_i \theta_{i} }{\theta_0 + \Theta_{S_t}} \bigg) \wedge 1\bigg) \indic\{\cE\} \bigg] \nonumber \\
    & \leq  \sum_{t=1}^T  \E\bigg[   \bigg( \bigg( \sum_{i\in S_t} \frac{r_i (\theta_{i,t}^{\ucb} - \theta_{i}) }{\theta_0 + \Theta_{S_t}} \bigg) \wedge 1\bigg) \indic\{\cE\} \bigg] \hspace*{2cm} \text{because } \Theta_{S_t,t}^{\ucb} \geq \Theta_{S_t} \text{ under $\cE$} \nonumber \\
    & \leq  \sum_{t=1}^T  \E\bigg[   \bigg( \bigg( \sum_{i\in S_t} \frac{ |\theta_{i,t}^{\ucb} - \theta_{i}| }{\theta_0 + \Theta_{S_t}} \bigg) \wedge 1\bigg) \indic\{\cE\} \bigg] \hspace*{2cm} \text{because } r_i \leq 1 \nonumber \\
    & \leq   \sum_{i=1}^K \E\Bigg[   \sum_{t=1}^T \bigg(   \frac{ | \theta_{i,t}^{\ucb}- \theta_{i}| }{\theta_0 + \Theta_{S_t}} \wedge 1 \bigg) \indic\{i \in S_t\} \indic\{\cE\}   \Bigg] \nonumber \\
    & \leq   138xm^2K \theta_{\max}^2 + \sum_{i=1}^K \E\Bigg[   \sum_{t=1}^T   \frac{ | \theta_{i,t}^{\ucb}- \theta_{i}| }{\theta_0 + \Theta_{S_t}}  \indic\{i \in S_t, \tau_{i,t} \geq 138 x (m+1)^2 \theta_{\max}^2 \} \indic\{\cE\}   \Bigg] \nonumber \\
    & \leq 138xm^2K \theta_{\max}^2 +  \sum_{i=1}^K \sqrt{ \sum_{t=1}^T \E \Bigg[ \frac{ \big(\frac{\theta_0}{m} + \theta_{i}\big) \indic\{i \in S_t\}}{\theta_0 + \Theta_{S_t}} \Bigg] } \nonumber \\
    & \hspace*{2cm} \times  \sqrt{\smash{\underbrace{\sum_{t=1}^T \E \Bigg[ \bigg( \frac{ | \theta_{i,t}^{\ucb}- \theta_{i}| }{\theta_0 + \Theta_{S_t}}\bigg)^2\frac{\theta_0 + \Theta_{S_t}}{\frac{\theta_0}{m} + \theta_i} \indic\{i \in S_t, \tau_{i,t} \geq 138 x (m+1)^2 \theta_{\max}^2 \} \indic\{\cE\} \Bigg] }_{=: A_T(i)} } \mystrut(20,10) } \mystrut(20,24) \label{eq:longeq}
\end{align}
where the last inequality is by Cauchy-Schwarz inequality. 
Now, the term $A_T(i)$ above may be upper-bounded as follows
\begin{align*}
A_T(i) & := \sum_{t=1}^T \E \Bigg[ \bigg( \frac{ | \theta_{i,t}^{\ucb}- \theta_{i}| }{\theta_0 + \Theta_{S_t}} \bigg)^2\frac{\theta_0 + \Theta_{S_t}}{\frac{\theta_0}{m} + \theta_i} \indic\{i \in S_t, \tau_{i,t} \geq 138 x (m+1)^2 \theta_{\max}^2\} \indic\{\cE\} \Bigg] \\
 & =   \E \Bigg[  \frac{ (\theta_{i,t}^{\ucb}- \theta_{i})^2 }{ \big( \frac{\theta_0}{m} + \theta_i\big) \theta_0 + \Theta_{S_t}  }  \indic\{i \in S_t, \tau_{i,t} \geq 138 x (m+1)^2 \theta_{\max}^2\} \indic\{\cE\} \Bigg]  \,.
\end{align*}
Now, since under the event $\cE$ by Lemma~\ref{lem:nibound}, $\tau_{i,t} \geq  138 x (m+1)^2 \theta_{\max}^2$ implies 
\[
    n_{i0,t} \geq 69 x (\theta_0 + \theta_i) (m+1) \theta_{\max} \geq 69 x (\theta_0 + \theta_i) \,.
\]
Therefore, we can apply Lemma~\ref{lem:tconc}, which further upper-bounds
\begin{align*}
 A_T(i) & \leq  \sum_{t=1}^T \E \Bigg[ \bigg( \frac{2^6 (\theta_0 +\theta_{i})^2 x }{n_{i0,t}} +  \frac{2 (22 x)^2 (\theta_0 + \theta_i)^4}{n_{i0,t}^2 (\frac{\theta_0}{m} + \theta_i)} \bigg) \times \frac{\indic\{i \in S_t, \tau_{i,t} \geq  138 x (m+1)^2 \theta_{\max}^2 \}}{\theta_0 + \Theta_{S_t}} \indic\{\cE\} \Bigg]    \\
 & \leq   \sum_{t=1}^T \E \Bigg[ \bigg( \frac{2^6 (\theta_0 +\theta_{i})^2 x }{n_{i0,t}} +  \frac{15 x (\theta_0 + \theta_i)^3}{n_{i0,t} \theta_{\max} (\theta_0 + m \theta_i)} \bigg) \times \frac{\indic\{i \in S_t\}}{\theta_0 + \Theta_{S_t}} \indic\{\cE\} \Bigg]  
 \end{align*}
 where we used $n_{i0,t} \geq 69 x (\theta_0 + \theta_i) m \theta_{\max}$ in the last inequality. Then, we get
 \begin{align*}
 A_T(i)  
 & \leq    \sum_{t=1}^T \E \Bigg[ \bigg( \frac{(\theta_0 +\theta_{i})^2 x }{n_{i0,t}} +  \frac{30 x (\theta_0 + \theta_i)}{n_{i0,t} } \bigg) \times \frac{\indic\{i \in S_t\}}{\theta_0 + \Theta_{S_t}} \indic\{\cE\} \Bigg] \\
    & \leq  (94 + 64 \theta_i)x   \sum_{t=1}^T \E \Bigg[ \frac{ (\theta_0 + \theta_i) \indic\{i \in S_t\}}{(\theta_0 + \Theta_{S_t}) n_{i0,t}}  \Bigg] \\
    & =  (94 + 64 \theta_i)x  \E \Bigg[ \sum_{t=1}^T \frac{  \indic\{i_t \in \{i,0\}, i\in S_t\}}{n_{i0,t}} \Bigg]\\
    & =    (94 + 64 \theta_i)x  \E \Big[ 1 + \log\big(n_{i0}(T)\big)\Big] \\
    & \leq  158 \theta_{\max} x  (1 + \log T)\,.
\end{align*}
Substituting into~\eqref{eq:longeq}, we then obtain using Cauchy-Schwarz inequality,
\begin{align*}
\sum_{t = 1}^T  & \E\Big[  \Big(\big( \cR(S_t,\theta_t^{\ucb}) - \cR(S_t, \theta) \big) \wedge 1\Big) \indic\{\cE\} \Big]  \\
    & \leq 138xm^2K \theta_{\max}^2 +  13 \sqrt{ \theta_{\max}x  (1 + \log T)} \sum_{i=1}^K \sqrt{ \sum_{t=1}^T \E \Bigg[ \frac{ \big(\frac{\theta_0}{m} + \theta_{i}\big) \indic\{i \in S_t\}}{\theta_0 + \Theta_{S_t}} \Bigg] } \\
    & \leq  138xm^2K \theta_{\max}^2 +  13\sqrt{\theta_{\max}x  (1 + \log T)} \sqrt{  \E \Bigg[  K  \sum_{t=1}^T \frac{ \sum_{i=1}^K  \big(\frac{\theta_0}{m} + \theta_{i}\big)  \indic\{i \in S_t\}}{\theta_0 + \Theta_{S_t}} \Bigg] } \\
    &  = 138xm^2K \theta_{\max}^2 +  13\sqrt{\theta_{\max} x  (1 + \log T) KT} \,.
\end{align*}
Finally, replacing into Inequality~\eqref{eq:regMNL1} yields
\[
     \regu_T  \leq 4T^2Ke^{-x} + 138xm^2K \theta_{\max}^2  +  13\sqrt{\theta_{\max} x  (1 + \log T) KT}  \,.
\]
Choosing $x = 2 \log T$ concludes the proof. 
\end{proof}

\subsection{Proof of Theorem~\ref{thm:thetamaxucb}}

The proof follows the one of Theorem~\ref{thm:utilm_wf}, except that the concentration lemmas should be generalized to any pairs $(i,j)$ instead of only with respect to item $0$, whose proofs are left to the reader and closely follows the one of Lemma~\ref{lem:tconc} and~\ref{lem:nibound}. For simplicity, this proof is performed up to universal multiplicative constants, using the rough inequality $\lesssim$.

\begin{lemma} 
\label{lem:gconc}
Let $T \geq 1$ and $x >0$. Then, with probability at least $1- 3K(K+1)T e^{-x}$, simultaneously for all $t \in [T]$ and $i \neq j$ in $[\tilde K]$: $\gamma_{ij} := \frac{\theta_i}{\theta_j} \leq \gamma^{\ucb}_{ij,t}$ and one of the following two inequalities is satisfied
\[
    n_{ij,t} < 69 x (1 +\gamma_{ij})
\qquad
\text{ or }
\qquad
\gamma_{ij,t}^{\ucb} \leq \gamma_{ij} + 4(\gamma_{ij} + 1) \sqrt{\frac{2 \gamma_{ij} x}{n_{ij,t}}} + \frac{22 x (\gamma_{ij} + 1)^2}{n_{ij,t}} \,.
\]
\end{lemma}

\begin{lemma} \label{lem:hatthetaucb}
Let $T \geq 1$ and $x >0$. Then, with probability at least $1- 3K(K+1)T e^{-x}$, simultaneously for all $t \in [T]$ and $i\in [K]$: $\hat \theta_{i,t}^{\ucb} := \min_{j} \gamma_{ij,t}^{\ucb} \gamma_{j0,t}^\ucb \geq \theta_{i}$ and for all $j$ one of the following two inequalities is satisfied
\[
    n_{ij,t} \lesssim  x (1 +\gamma_{ij})
\qquad
\text{ or }
\qquad
    n_{j0,t} \lesssim  x (1 + \theta_j)^2\theta_j^{-1}
\]
or 
\begin{equation*}
\gamma_{ij,t}^{\ucb} \gamma_{j0,t}^{\ucb} - \theta_i \lesssim   \sqrt{(\gamma_{ij} + 1) \theta_i x }  \bigg( \sqrt{\frac{(\theta_i + \theta_j)}{n_{ij,t}}} + \sqrt{\frac{(1 + \theta_j)}{n_{j0,t}}}\bigg)   + (\gamma_{ij} + 1) \frac{(\theta_i + \theta_j)x}{n_{ij,t}} + \frac{\gamma_{ij} (1+\theta_j)^2x }{n_{j0,t}}
\,.
\end{equation*}
\end{lemma}

\begin{proof}[Proof of Lemma~\ref{lem:hatthetaucb}] 
The proof follows from Lemma~\ref{lem:gconc}. If $n_{ij,t} > Cx(1+\gamma_{ij})$ and $n_{j0,t} > Cx(1+\theta_j)$ for some large enough constant C, we have
\[
\gamma_{ij,t}^{\ucb} \leq \gamma_{ij} + 4(\gamma_{ij} + 1) \sqrt{\frac{2 \gamma_{ij} x}{n_{ij,t}}} + \frac{22 x (\gamma_{ij} + 1)^2}{n_{ij,t}} 
\]
and
\[
    \gamma_{j0,t}^{\ucb} \leq \gamma_{j0} + 4(\gamma_{j0} + 1) \sqrt{\frac{2 \gamma_{j0} x}{n_{j0,t}}} + \frac{22 x (\gamma_{j0} + 1)^2}{n_{j0,t}}  \leq 2 \gamma_{j0} \,.
\]
This implies,
\begin{align*}
    \gamma_{ij,t}^{\ucb} \gamma_{j0,t}^{\ucb} - \theta_i
        & = \gamma_{ij,t}^{\ucb} \gamma_{j0,t}^{\ucb} - \gamma_{ij} \gamma_{j0} =  (\gamma_{ij,t}^{\ucb} - \gamma_{ij}) \gamma_{j0,t}^{\ucb} + \gamma_{ij}(\gamma_{j0,t}^{\ucb} - \gamma_{j0})  \\
        & \leq 2 (\gamma_{ij,t}^{\ucb} - \gamma_{ij}) \gamma_{j0} + \gamma_{ij}(\gamma_{j0,t}^{\ucb} - \gamma_{j0})   \\
        & \leq  8 \gamma_{j0} (\gamma_{ij} + 1) \sqrt{\frac{2 \gamma_{ij} x}{n_{ij,t}}} +  \frac{44 x \gamma_{j0} (\gamma_{ij} + 1)^2}{n_{ij,t}}  +  4\gamma_{ij} (\gamma_{j0} + 1) \sqrt{\frac{2 \gamma_{j0} x}{n_{j0,t}}} + \frac{22 x \gamma_{ij} (\gamma_{j0} + 1)^2}{n_{j0,t}}\Big)\,.
\end{align*}
Replacing $\gamma_{ij} = \theta_i / \theta_j$ and $\gamma_{j0} = \theta_j$ concludes the proof. 
\end{proof}

\begin{lemma} \label{lem:niboundij}
Let $T \geq 1$ and $x >0$. Then, with probability at least $1-K(K+1)Te^{-x}$
\begin{equation}
    \label{eq:niboundij}
    \tau_{ij,t} < 2 x\frac{(\theta_0 + \Theta_{S^*})^2}{\theta_i + \theta_j} \  \text{ or } \  n_{ij,t} \geq \frac{(\theta_i + \theta_j) \tau_{ij,t}}{2(\theta_0 + \Theta_{S^*})}\,,
\end{equation}
where $\tau_{ij,t} := \sum_{s=1}^{t-1} \indic\{\{i,j\} \subseteq S_s\}$ simultaneously for all $t \in [T]$ and $i \neq j \in [K]$.
\end{lemma}

\begin{proof}[Proof of Theorem~\ref{thm:thetamaxucb}]
Let $\cE$ be the high-probabality event of Lemmas~\ref{lem:hatthetaucb} and~\ref{lem:niboundij} are satisfied, so that  $\P(\cE) \geq 1-4K^2Te^{-x}$.
First, note that since we have under the event $\cE$, $\hat \theta_t^{\ucb} \leq \theta_t^{\ucb}$, our procedure also satisfies the regret upper-bound
\[
    \regu_T \leq O(\sqrt{\theta_{\max} KT} \log T)
\]
of Theorem~\ref{thm:utilm_wf}. Indeed, all upper-bounds of the proof of Theorem~\ref{thm:utilm_wf} remain valid upper-bounds except the probability of the event $\cE^c$ which is $O(T^{-1})$ for $x = 2\log T$. 

Let us now prove that we also have $R_T \leq O(K \sqrt{T} \log T)$ with no asymptotic dependence on $\theta_{\max}$ when $T \to \infty$.

Then, 
\begin{align}
    \regu_T  & = \sum_{t=1}^T \E\big[\cR(S^*, \theta) - \cR(S_t, \theta)\big] 
    = \sum_{t=1}^T \E\big[(\cR(S^*, \theta) - \cR(S_t, \theta)) \indic\{\cE\} + (\cR(S^*, \theta) - \cR(S_t, \theta)) \indic\{\cE^c\}\big]  \nonumber \\
    & \leq \sum_{t = 1}^T \E\Big[ \big( ( \cR(S_t,\hat \theta_{t}^{\ucb}) - \cR(S_t, \theta)) \wedge \cR(S^*, \theta) \big) \indic\{\cE\} + \cR(S^*, \theta) \indic\{\cE^c\} \Big] \nonumber\,.
\end{align}
Then, using $\cR(S^*, \theta) \leq \max_i r_i \leq 1$, we get
\begin{align}
    \regu_T  &  \leq \sum_{t = 1}^T \E\Big[ \big( (\cR(S_t,\hat \theta_{t}^{\ucb}) - \cR(S_t, \theta)) \wedge 1 \big) \indic\{\cE\} + \indic\{\cE^c\} \Big]  \nonumber \\
    & \leq  4T^2K(K+1)^2e^{-x} + \sum_{t = 1}^T \E\Big[ \Big( \big(   \cR(S_t,\hat \theta_{t}^{\ucb}) - \cR(S_t, \theta) \big) \wedge 1\Big) \indic\{\cE\} \Big] \label{eq:regMNL11} \,.
\end{align}
%


Follow the proof of Theorem~\ref{thm:utilm_wf}, we upper-bound the second term of the right-hand-side of~\eqref{eq:regMNL11}:
\begin{align}
 \sum_{t = 1}^T&  \E\Big[  \Big(\big( 
 \cR(S_t, \hat \theta_{t}^{\ucb}) - \cR(S_t, \theta) \big) \wedge 1\Big) \indic\{\cE\} \Big] \\
    & = \sum_{t=1}^T  \E\bigg[   \bigg( \bigg( \min_{j \in [K]} \sum_{i\in S_t} \frac{r_i \hat \theta_{i,t}^{\ucb} }{1 + \sum_{j\in S_t} \hat \theta_{j,t}^{\ucb}} - \frac{r_i \theta_{i} }{1 + \sum_{j\in S_t} \theta_{j}} \bigg) \wedge 1\bigg) \indic\{\cE\} \bigg] \nonumber \\
    & \leq  \sum_{t=1}^T  \E\bigg[   \bigg( \bigg( \sum_{i\in S_t} \frac{r_i (\hat \theta_{i,t}^{\ucb} - \theta_{i}) }{\theta_0 + \Theta_{S_t}} \bigg) \wedge 1\bigg) \indic\{\cE\} \bigg] \hspace*{2cm} \text{because } \sum_{i \in S_t} \hat \theta_{i,t}^{\ucb} \geq \Theta_{S_t} \text{ under $\cE$} \nonumber \\
    & \leq  \sum_{t=1}^T  \E\bigg[   \bigg( \bigg( \sum_{i\in S_t} \frac{ |\hat \theta_{i,t}^{\ucb} - \theta_{i}| }{\theta_0 + \Theta_{S_t}} \bigg) \wedge 1\bigg) \indic\{\cE\} \bigg] \hspace*{2cm} \text{because } r_i \leq 1 \nonumber \\
    & \leq   \sum_{i=1}^K \E\Bigg[   \sum_{t=1}^T \bigg(   \frac{ | \hat \theta_{i,t}^{\ucb} - \theta_{i}| }{\theta_0 + \Theta_{S_t}} \wedge 1 \bigg) \indic\{i \in S_t\} \indic\{\cE\}   \Bigg] \nonumber \\
    & \leq \sum_{i=1}^K \E\Bigg[   \sum_{t=1}^T \bigg(   \frac{ | \gamma_{ij_t,t}^{\ucb}\gamma_{j_t0,t}^{\ucb} - \theta_{i}| }{\theta_0 + \Theta_{S_t}} \wedge 1 \bigg) \indic\{i \in S_t\} \indic\{\cE\}   \Bigg] \nonumber
\end{align}
where $j_t = \argmax_{j \in S_t\cup \{0\}} \theta_j$, where the last inequality is by definition of $\hat \theta_{i,t}^{\ucb}$. Now,  from Lemma~\ref{lem:hatthetaucb}, paying an additive exploration cost to ensure that $n_{ij,t} \gtrsim x (1+\gamma_{ij})$ and $n_{j0,t} \gtrsim x(1+\theta_{j})^2 \theta_{j}$ for all $j \in S_t$ such that $\theta_j \geq \theta_0$. From Lemma~\ref{lem:niboundij}, this is satisfied if for some constant $C >0$
\[
    \tau_{ij,t} > C m^2 \theta_{\max}^2x \,.
\]
Such a condidtion can be wrong for a couple $(i,j) \in S_t^2$ at most during $C K^2 m^2 \theta_{\max}^2x = O(\log T)$ rounds (since $\tau_{ij,t}$ increases then). Thus, for $C$ large enough,

\begin{align}
 \sum_{t = 1}^T&  \E\Big[  \Big(\big( 
 \cR(S_t, \hat \theta_{t}^{\ucb}) - \cR(S_t, \theta) \big) \wedge 1\Big) \indic\{\cE\} \Big] \nonumber \\
    & \leq   O(\log T) + \sum_{i=1}^K \E\Bigg[   \sum_{t=1}^T   \frac{ | \gamma_{ij_t,t}^{\ucb}\gamma_{j_t0,t}^{\ucb} - \theta_{i}| }{\theta_0 + \Theta_{S_t}}  \indic\{i \in S_t, \tau_{ij_t,t}\wedge\tau_{j_t,t} \geq C x m^2 \theta_{\max}^2  \} \indic\{\cE\}   \Bigg] \nonumber \\
    & \lesssim O(\log T) + \sum_{i=1}^K \E\Bigg[   \sum_{t=1}^T   \bigg( \sqrt{(\gamma_{ij_t} + 1) \theta_i x }  \bigg( \sqrt{\frac{(\theta_i + \theta_{j_t})}{n_{ij_t,t}}} + \sqrt{\frac{(1 + \theta_j)}{n_{j_t0,t}}}\bigg) \nonumber \\
    & \hspace*{5cm} + (\gamma_{ij_t} + 1) \frac{(\theta_i + \theta_{j_t})x}{n_{ij_t,t}} + \frac{\gamma_{ij_t} (1+\theta_{j_t})^2x }{n_{j_t0,t}}\bigg)   \frac{\indic\{i \in S_t \}}{{\theta_0 + \Theta_{S_t}}}   \Bigg] \nonumber\\
    & \leq O(\log T) + \sum_{i=1}^K \E\Bigg[   \sum_{t=1}^T    \sqrt{(\gamma_{ij_t} + 1) \theta_i x }  \bigg( \sqrt{\frac{(\theta_i + \theta_{j_t})}{n_{ij_t,t}}} + \sqrt{\frac{(1 + \theta_{j_t})}{n_{j_t0,t}}}\bigg)  \frac{\indic\{i \in S_t \}}{{\theta_0 + \Theta_{S_t}}}   \Bigg] \nonumber
\end{align}
where the last inequality is because using that $\{i, j_t, 0\} \subseteq S_t$, we have
\begin{align*}
\E\bigg[ \sum_{t=1}^T \frac{1 + \theta_{j_t} }{(1 + \Theta_{S_t}) n_{j_t0,t}} \bigg]  = \E \bigg[  \sum_{t=1}^T \sum_{j=1}^K \frac{\indic\{i_t \in \{j,0\} \}}{n_{j0,t}}  \indic\{j = j_t\} \bigg] \leq K(1 + \log T). 
\end{align*}
and 
\begin{align*}
\E\bigg[ \sum_{t=1}^T \frac{\theta_i + \theta_{j_t} }{(1 + \Theta_{S_t}) n_{ij_t,t}} \bigg]  = \E \bigg[  \sum_{t=1}^T \sum_{j=1}^K \frac{\indic\{i_t \in \{j,i\}\}}{n_{j0,t}}  \indic\{j = j_t\} \bigg] \leq K(1 + \log T). 
\end{align*}
Then, by Cauchy-Schwarz inequality we further get
\begin{align}
 \sum_{t = 1}^T&  \E\Big[  \Big(\big( 
 \cR(S_t, \hat \theta_{t}^{\ucb}) - \cR(S_t, \theta) \big) \wedge 1\Big) \indic\{\cE\} \Big] \nonumber \\
    & \lesssim O(\log T) + \sum_{i=1}^K  \sqrt{ \E\Bigg[    \sum_{t=1}^T  \frac{ (\gamma_{ij_t} + 1) \theta_i \indic\{i \in S_t \} x }{\theta_0 + \Theta_{S_t}} \Bigg]  } 
    \sqrt{\E\Bigg[ \sum_{t=1}^T \bigg( \frac{(\theta_i + \theta_{j_t})}{n_{ij_t,t}} +\frac{(1 + \theta_{j_t})}{n_{j_t0,t}} \bigg)  \frac{\indic\{i \in S_t \}}{{\theta_0 + \Theta_{S_t}}} \Bigg] }    \nonumber \\
    & \lesssim O(\log T) + \sum_{i=1}^K  \sqrt{ \E\Bigg[    \sum_{t=1}^T  \frac{ (\gamma_{ij_t} + 1) \theta_i \indic\{i \in S_t \} x }{\theta_0 + \Theta_{S_t}} \Bigg]  } 
    \sqrt{K \log T}    \nonumber \\
    & \lesssim O(\log T) + \sum_{i=1}^K  \sqrt{ \E\Bigg[    \sum_{t=1}^T  \frac{ \theta_i \indic\{i \in S_t \} x }{\theta_0 + \Theta_{S_t}} \Bigg]  } 
    \sqrt{K \log T}     \qquad \text{because $\gamma_{ij_t} \leq 1$ by definition of $j_t$} \nonumber \\
    & \leq O(K \sqrt{T x \log T}) = O(K\sqrt{T} \log T) \,,
\end{align}
where the last inequality is by Jensen's inequality and the equality by setting $x = 2 \log T$ to control the probability that $\cE^c$ occurs. This concludes the proof.
\end{proof}

}

\end{document}